\documentclass[letterpaper]{article} 
\usepackage{aaai2026}  
\nocopyright
\usepackage{times}  
\usepackage{helvet}  
\usepackage{courier}  
\usepackage[hyphens]{url}  
\usepackage{graphicx} 
\usepackage{natbib}  
\usepackage{caption} 
\usepackage{algorithm}
\usepackage{algorithmic}

\usepackage{newfloat}
\usepackage{listings}
\usepackage{amsmath}
\usepackage{amsthm}
\usepackage{booktabs}
\usepackage{amsfonts}
\usepackage{comment}
\usepackage{color}
\usepackage{subfig}
\usepackage{url}
\newtheorem{theoremrestate}{Theorem}

\newcommand{\model}{CtrTab}

\DeclareCaptionStyle{ruled}{labelfont=normalfont,labelsep=colon,strut=off} 
\floatstyle{ruled}
\newfloat{listing}{tb}{lst}{}
\floatname{listing}{Listing}
\title{Towards Synthesizing High-Dimensional Tabular Data with Limited Samples}
\author {
    Zuqing Li\textsuperscript{\rm 1},
    Junhao Gan\textsuperscript{\rm 1},
    Jianzhong Qi\textsuperscript{\rm 1}\thanks{Corresponding author.}
}
\affiliations {
    \textsuperscript{\rm 1}School of Computing and Information Systems, The University of Melbourne\\
    \{zuqingl@student, junhao.gan@, jianzhong.qi@\}unimelb.edu.au
}
\usepackage{bibentry}

\begin{document}

\maketitle

\begin{abstract}
Diffusion-based tabular data synthesis models have yielded promising results. However, when the data dimensionality increases, existing models tend to degenerate and may perform even worse than simpler, non-diffusion-based models. This is because limited training samples in high-dimensional space often hinder generative models from capturing the distribution accurately. 
To mitigate the insufficient learning signals and to stabilize training under such conditions, we propose \model, a condition-controlled diffusion model that injects perturbed ground-truth samples as auxiliary inputs during training. This design introduces an implicit $L_2$ regularization on the model's sensitivity to the control signal, improving robustness and stability in high-dimensional, low-data scenarios. Experimental results across multiple datasets show that \model\ outperforms state-of-the-art models, with a performance gap in accuracy over 90\% on average.
\end{abstract}
\begin{links}
    \link{Code}{https://github.com/zuqingli0404/CtrTab}
\end{links}
\section{Introduction}

Tabular data synthesis is an important problem with a wide range of applications. 
A common motivation is to facilitate privacy-preserving data sharing, i.e., to use synthetic data in scenarios where access to real data is restricted due to privacy concerns. In recent years, tabular data synthesis has also been used to help address data scarcity~\citep{dallm,relddpm,survey2}, augmenting training datasets to satisfy the need of modern machine learning models which are often data hungry. Meanwhile, the database community is using synthesized data for system performance benchmarking~\cite{clavaddpm,DBMSsynthesis,sam}.
In this work, we focus on the non-privacy-sensitive settings, and we aim to synthesize data to enhance the performance of downstream tasks such as machine learning effectiveness~\citep{dataaugment}.

Early studies on tabular data synthesis are primarily based on statistical models~\citep{condensation,Fourierdecompositions,DPSynthesizer,gibbssamplers,Privbayes}. With the rise of deep learning, models based on GANs and diffusion~\citep{itsgan,octgan,tabddpm,relddpm,ctgan} are adopted. At the same time, to ensure the quality of synthesized data, studies introduce conditional generation~\citep{relddpm,ctgan,ctabgan+}, incorporating additional information to guide the  synthesis process. 

Existing studies~\citep{stasy,tabddpm,codi,relddpm,tabsyn} focus on datasets with a small number of columns (typically fewer than 50) and a large number of samples, where learning the underlying distribution is more tractable. 
A critical challenge underexplored is the difficulty posed by sparse, high-dimensional tabular data, i.e., tables with many (e.g., hundreds of) columns and only a few rows. 
\begin{figure}[t]
    \centering
    \includegraphics[width=0.7\linewidth]{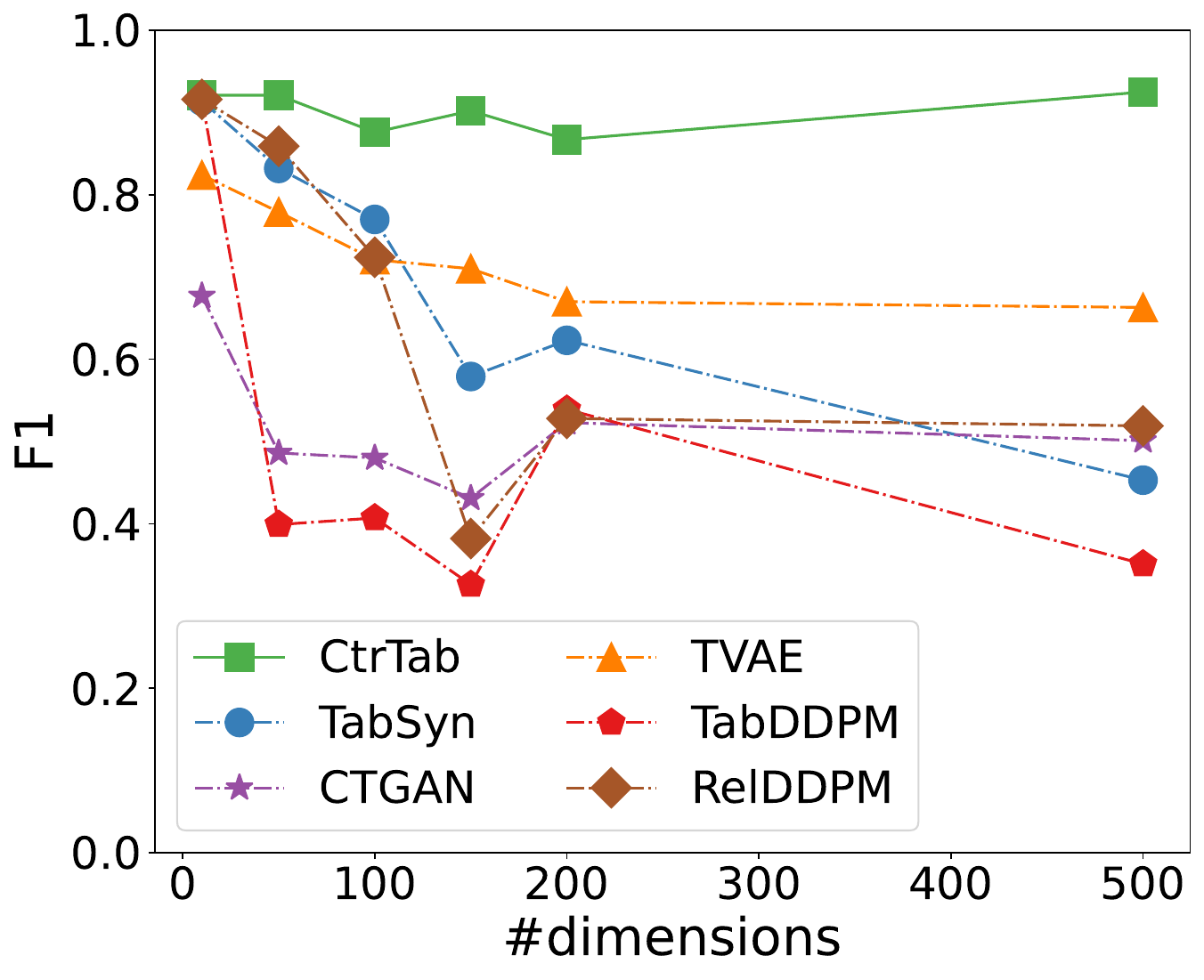}
    \caption{Challenges in tabular data synthesis over high-dimensional data. As the dimensionality increases, F1 scores of all existing models in machine learning tests decrease, while those of our model \model\ remain stable.}
    \label{fig:tabulardatageneration}
\end{figure} 

This situation, common in fields like biomedicine, suffers from the curse of dimensionality — as the number of features grows, data points become increasingly sparse in the feature space, and the distances between neighboring samples grow larger, making it  difficult for models to accurately capture the underlying distribution.
In addition, the limited number of samples poses the risk of overfitting, as models may memorize the training data instead of generalizing to unseen samples, further aggravating the challenge of reliable data synthesis in high-dimensional settings.
Even the performance of the state-of-the-art (SOTA) diffusion model TabSyn~\citep{tabsyn} for tabular data synthesis is still far from satisfactory under such settings. To verify this, we use data generation tools from \texttt{Scikit-learn} to generate tabular data, fixing the number of samples (rows) at 3,000 while varying the number of features (dimensions) from 10 to 500, with a class balance ratio of  0.5 for a binary classification task.
For each dataset, we use 80\% data to train recent tabular data synthesis models, and 20\% for testing (following the machine learning tests detailed in the experimental section). Figure~\ref{fig:tabulardatageneration} plots the test results in F1 score. There is an overall decreasing trend in F1 for all models (except \model\ which is ours), which is particularly obvious as the dimensionality reaches 500.

To address this issue, we propose \textbf{\model}, a condition-controlled diffusion model designed for high-dimensional, low-data regimes.
\model\ introduces a control module alongside the denoising network.
During training, each sample is perturbed with Laplace noise, serving as an auxiliary control signal.
This control input is encoded and integrated into the decoder of the denoising network.
Beyond improving controllability during generation, the control module provides an additional learning signal during training, enhancing the model's robustness.

For better generalization under limited training samples, we employ a noise injection training strategy that systematically perturbs training samples.
We theoretically show that this process is equivalent to introducing an implicit $L_2$ regularization on the model’s sensitivity to the control input, promoting smoother mappings and enhancing generalization beyond the sparse training data~\citep{noiseaccumulation}.

We conduct extensive experiments to show that while simple designs, e.g., perturbing data for augmentation or expanding model capacity, could offer some improvements,
they fall short of addressing the underlying challenges.
In contrast, \model\ outperforms state-of-the-art models across various datasets, validating the effectiveness of our design.

Our contributions are summarized as follows:
(1)~We propose a diffusion-based tabular data synthesis model named \model\ to address the challenge of sparse, high-dimensional data. 
    Unlike existing diffusion-based models~\cite{tabddpm,tabsyn}  which struggle significantly under such settings, \model\ introduces a \textit{control module} and a \textit{noise injection training strategy} that work together to improve model generalizability and robustness in complex tabular scenarios. 
(2)~We provide a theoretical analysis showing that the proposed noise-based training is equivalent to $L_2$ regularization, where the noise scale flexibly controls the strength of regularization, enhancing model smoothness and stability.
(3)~We conduct experiments that extend tabular data synthesis to tables with up to 10,001 dimensions.
 The results show that machine learning models trained with data synthesized by \model\ are much more accurate than those trained with data synthesized by SOTA models, with a performance gain of over 90\% on average, confirming that \model\ is more effective in learning the data distribution.  

\section{Related Work}

\paragraph{Tabular Data Synthesis} 
Recent advances in tabular data synthesis have shifted from statistical approaches~\citep{condensation,Fourierdecompositions,DPSynthesizer,gibbssamplers,Privbayes} to deep generative models such as GANs~\citep{itsgan,octgan,tablegan,causaltgan,ctgan,ctabgan+}, large language models (LLMs)~\citep{MaskedLanguageModeling,HARMONIC} and diffusion models~\citep{stasy,tabddpm,codi,relddpm,clavaddpm,tabdiff,tabsyn,TabRep}. While GAN-based models suffer from training instability~\citep{wgan}, diffusion- and LLM-based models~\citep{MaskedLanguageModeling,HARMONIC} have shown improved sample quality and training robustness. Recent studies also consider missing-value imputation or minority-class data synthesis~\citep{imbalancedclassification,TabularDataImputation,Imb-FinDiff}. 

Most existing diffusion models focus on low-dimensional, dense tabular data and do not generalize well to high-dimensional, low-sample scenarios. 
Our model represents a latest development of the diffusion-based models that addresses this gap.

\paragraph{Conditional Tabular Data Synthesis}
Conditional generative models have been proposed for tabular data synthesis with additional constraints.  CTGAN~\citep{ctgan} incorporates class labels in the generator to produce data conditioned on the target class. CTAB-GAN+~\citep{ctabgan+} extends CTGAN to support both discrete and continuous data class labels. 
RelDDPM~\citep{relddpm} uses a classifier-guided \textit{conditional diffusion model}. It first trains an unconditional model to fit the input data distribution. Then, given a constraint, e.g., a target class label, it trains a classifier, the gradient of which is used to control sample generation. 
Our model does \emph{not} concern class-conditioned generation. Instead, we introduce a control module to guide learning under sparse, high-dimensional conditions. There are also works utilizing LLMs~\citep{resrag,HARMONIC}. We aim for lighter-weight solutions and do not consider such solutions further.

\paragraph{Learning High-Dimensional Distribution with Sparse Data}
High-dimensional data poses significant challenges for machine learning.
In \textit{probably approximately correct}  learning~\citep{paclearning}, the generalization error of a model depends on both the data dimensionality and the hypothesis class complexity.  
Typically, both the hypothesis space complexity and the required sample size grow exponentially as the dimensionality increases, making generalization more difficult. 
For generative models, increased dimensionality results in sparse data distributions \citep{patternrecognition}, making it difficult for models to capture the underlying data structure.

\section{Preliminaries}

\paragraph{Problem Statement}

Consider a table $\mathcal{T}$ of $N$ rows, where each row $x_{raw}$ consists of $D_\text{{num}}$ numerical features and $D_\text{{cat}}$ 
categorical features that correspond to variables $\mathbf{x}_{\text{num}}$ and $\mathbf{x}_{\text{cat}}$, respectively. 
Categorical variables are encoded into numerical representations (e.g., one-hot encoding) before being used as model input. Let the dimensionality of all categorical variables after encoding be $E(D_\text{{cat}})$. The total model input data dimensionality is then $D = E(D_\text{{cat}}) + D_\text{{num}}$. 
Each row of $\mathcal{T}$, i.e., a data sample, is expressed as $\mathbf{x} = [\mathbf{x}_\text{{num}}, \mathbf{x}_\text{{cat}}]$, where $\mathbf{x}_\text{{num}} \in \mathbb{R}^{D_\text{{num}}}$ and $\mathbf{x}_\text{{cat}}\in \mathbb{R}^{E(D_\text{{cat}})}$. 

Our aim is to train a generative model $p_{\theta}(\mathcal{T})$ such that the distribution of the generated data approximates that of the real data (i.e., rows) in $\mathcal{T}$. We focus on sparse, high-dimensional data, where $N \ll 2^D$. Here, $2^D$ represents the volume of the data space, which grows exponentially with $D$. We design our model based on diffusion models. Below, we briefly outline the core idea of diffusion models.

\paragraph{Denoising Diffusion Probabilistic Model}\label{sec:ddpm} 

The denoising diffusion probabilistic model (DDPM)~\citep{DDPM} (see Figure~\ref{fig:ddpm} in Appendix A) has two processes: a forward process and a reverse process, both Markov chains. 
In the \textbf{forward process}, noise is added to a data sample $\mathbf{x}_{0}$ from distribution $q(\mathbf{x}_{0})$:
{\small
\begin{align}
    & q(\mathbf{x_{t}}|\mathbf{x}_{{t-1}}) = \mathcal{N}(\mathbf{x_{{t}}};\sqrt{1-\beta_{t}}\mathbf{x_{{t-1}}},\beta_{t}\mathbf{I})\,, \quad \\
    &q(\mathbf{x_{{t}}}|\mathbf{x_{0}}) = \mathcal{N}(\mathbf{x_{t}};\sqrt{\Bar{\alpha_{t}}}\mathbf{x_{0}},(1-\Bar{\alpha_{t}})\mathbf{I}) \label{eq:ddpmforward}\,,
\end{align}}
where $t$ is the number of noise adding steps (timesteps), $\beta_{t}$ controls the variance of the noise, $\alpha_{t} = 1 - \beta_{t}$, and $ \Bar{\alpha_{t}} = \prod\nolimits_{i=1}^{t} \alpha_{t}$. 
In the \textbf{reverse process}, given a sample $\mathbf{x}_{t}$ with pure random noise which is usually from a Gaussian distribution,
a \emph{denoising network} (i.e., the diffusion model) learns to iteratively denoise $\mathbf{x}_{t}$ until it is restored to the initial sample $\mathbf{x}_{0}$. The reverse process can also be expressed as a Gaussian distribution through derivation:
{\small
\begin{align}
    q(\mathbf{x_{{t-1}}}|\mathbf{x}_{t},\mathbf{x}_0) =  \mathcal{N}(\mathbf{x}_{{t-1}}; \mathbf{\tilde \mu_{t}}(\mathbf{x}_{t}, \mathbf{x}_0), \tilde\beta_{t}) \label{eq:ddpmbackward}\,,
\end{align}
}
where $\mathbf{\tilde\mu}_{t}(\mathbf{x}_{t}, \mathbf{x}_0) = \frac{\sqrt{\bar\alpha_{{t-1}}}\beta_{t} }{1-\bar\alpha_{t}}\mathbf{x}_0 + \frac{\sqrt{\alpha_{t}}(1- \bar\alpha_{{t-1}})}{1-\bar\alpha_{t}} \mathbf{x}_{t}$ and  $\tilde\beta_{t} = \frac{1-\bar\alpha_{{t-1}}}{1-\bar\alpha_{t}}\beta_{t}$.

The model learning process aims to maximize the variational lower bound:
{\small
\begin{align}
    &\log(q(\mathbf{x})) \geq {E}_\text{q} \bigg[ \underbrace{\log p_\theta(\mathbf{x}_0|\mathbf{x}_1)}_{L_0}- \underbrace{D_{KL}({q(\mathbf{x}_{T}|\mathbf{x}_0)}\|{p(\mathbf{x}_{T})})}_{L_{T}}  - \nonumber \\
    &\sum_{t > 1} \underbrace{D_{KL}({q(\mathbf{x}_{{t-1}}|\mathbf{x}_{t},\mathbf{x}_0)}\|{p_\theta(\mathbf{x}_{{t-1}}|\mathbf{x}_{t})})}_{L_{{t-1}}} \bigg]\,, \label{eq:elbo}
\end{align}
}where $D_{KL}(\cdot)$ is the KL divergence, 
$q(\mathbf{x})$ the probability distribution of $\mathbf{x}_0$, $p_\theta$  the parameterized model $\theta$ to approximate the probability distribution of the reverse process, and $T$ the timestep at which noise is added, such that the original data distribution matches standard Gaussian distribution.

To maximize this lower bound means to minimize the KL divergence between $q(\mathbf{x}_{{t-1}}|\mathbf{x}_{t},\mathbf{x}_0)$ and $p_\theta(\mathbf{x}_{{t-1}}|\mathbf{x}_{t})$.
This objective further reduces to minimizing the sum of mean-squared errors between $\epsilon$ --- the ground-truth error from $\mathbf{x}_0$ to $\mathbf{x}_t$
 --- and $\epsilon_\theta$, the predicted noise by a denoising network: 
{\small
\begin{align}
    L_{\text{simple}}(\theta) = \mathbb{E}_{\mathbf{x}_0, t, \mathbf{\epsilon}}\left[ \| \mathbf{\epsilon} - \mathbf{\epsilon}_\theta(\mathbf{x}_{t}, t) \|^2 \right] .\label{eq:lsimple}
\end{align}}

Here, $t$ and $\mathbf{\epsilon}$ are generated randomly at training, while $\mathbf{x}_{{t}}$ is from adding noise to $\mathbf{x}_0$ with Eq.~\eqref{eq:ddpmforward}. 

Once trained, the denoising network $\epsilon_\theta$ is used for data generation (called the sampling process). It takes $\mathbf{x}_{{t}}$ and a timestep $t$ as input, and its goal is to denoise $\mathbf{x}_{t}$ iteratively back to $\mathbf{x}_0$ (i.e., a generated sample). Specifically, at each timestep $t$, it computes $\mathbf{x}_{{t-1}}$ from $\mathbf{x}_{t}$ as follows:
{\small
\begin{align}
     \mathbf{x}_{t-1} = \frac{1}{\sqrt{\alpha_t}}(\mathbf{x}_t - \frac{1 - \alpha_t}{\sqrt{1 - \bar{\alpha_t}}}\epsilon_{\theta}(\mathbf{x}_t,t)) + \sigma_t \mathbf{z}\,,
\end{align}}
where $\sigma_{t}$ is the \emph{noise scale} at timestep $t$, and $\mathbf{z}$ follows the standard Gausssian distribution.

\paragraph{Score-based Diffusion Generative Model} The initial diffusion model uses discrete timesteps. \emph{Score-based generative models}~\citep{scoresde} further introduce a continuous-time formulation using Stochastic Differential Equations (SDEs), where the forward and reverse processes are guided by a \emph{score function} $\nabla_{\mathbf{x}}\log p_t(\mathbf{x})$ as follows:
{\small
\begin{align}
    &{\rm d} \mathbf{x} =  \mathbf{f}(\mathbf{x}, t){\rm d}t + g(t) \; {\rm d} \mathbf{w}_t \,, \quad \nonumber \\
    &{\rm d} \mathbf{x}  = [\mathbf{f}(\mathbf{x}, t) - g^2(t) \nabla_{\mathbf{x}}\log p_t(\mathbf{x})] {\rm d}t  + g(t) \;{\rm d} \mathbf{w}_t \,,
\end{align}
}
where $\mathbf{f}$ and $g$ denote the drift and diffusion coefficients, respectively, and $\mathbf{w}$ is the standard Wiener process. The diffusion process is typically categorized into variance-preserving (VP) and variance-exploding (VE) types.
Our solution supports VE, with a training objective to minimize~\citep{tabsyn}:
{\small
\begin{align}
    &L(\theta) = \mathbb{E}_{\mathbf{x}_0} \mathbb{E}_{\mathbf{x}_{t}|\mathbf{x}_0} \Vert D_{\theta}(\mathbf{x}_{t}, t) - \nabla_{\mathbf{x}_{t}}\log p(\mathbf{x}_{t}|\mathbf{x}_0) )\Vert_2^2  \nonumber \\
    &\approx \mathbb{E}_{\mathbf{x}_0} \mathbb{E}_{\mathbf{x}_{t}|\mathbf{x}_0}  \mathbf{\Vert {\epsilon}}_{\theta}(\mathbf{x}_{t}, t) - {{\mathbf{\varepsilon}}} \Vert_2^2\,.
\end{align}} 

Like DDPM, the denoising network here predicts the noise at $t$. Our model thus can support both DDPM and score-based SDE, upon which existing diffusion-based tabular data synthesis models are built. We use ``diffusion model'' hereafter to refer to both types when the context is clear.

\section{Methodology}\label{sec:method}

As Figure~\ref{fig:ctddpms} shows,  \model\ consists of two branches, a denoising network and a control module. The control module can be applied to different diffusion-based models -- we use the SOTA model TabSyn~\cite{tabsyn}. With this module, \model\ is trained on noisy data in addition to raw input to learn the input distribution more effectively, as will be shown by our theoretical analysis in the next section.

\paragraph{Denoising Network}

The denoising network takes as input a noisy sample \( \mathbf{x}_t \) 
(which comes from an encoded representation of ground-truth sample \( \mathbf{x}_0 \) during training and from standard Gaussian noise during sampling) 
and a timestep \( t \). It outputs a predicted noise \( \epsilon_\theta \). 

We adopt TabSyn's denoising network design and latent encoding strategy (detailed in Appendix A as this is not our focus). Each raw sample $x_{raw}$ from table $\mathcal{T}$ is transformed into \( \mathbf{x}_0 \) using a model-agnostic embedding method.
The hidden layers consist of three fully connected layers with SiLU activation functions. 
A final linear layer is applied to predict the noise.
Note that our contributions focus on the control module and noise-based training (detailed next). The denoising network is modular and can be replaced with other networks (e.g., RelDDPM~\cite{relddpm} in our experiments).

\begin{figure}[ht]
    \centering
    \includegraphics[width=\linewidth]{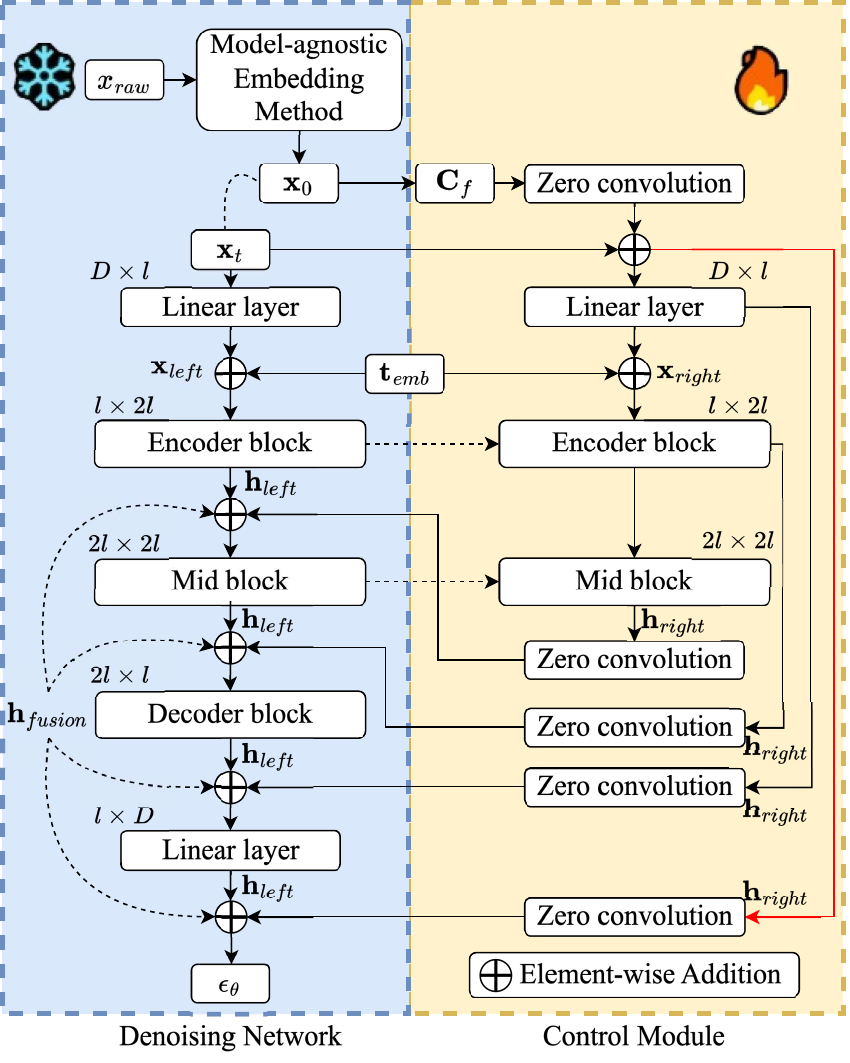}
    \caption{
        Overview of \model. The left (blue) is a denoising network, which receives the noisy input \( \mathbf{x}_t \) and the timestep \( t \), and predicts noise \( \epsilon_\theta \). 
        The right (yellow) is the control module, which encodes conditioning input \( \mathbf{C}_f \) and injects intermediate features via element-wise addition (\( \oplus \)). 
        All fusion operations are followed by zero convolution to match dimensions. 
        This modular design allows injecting conditioning signals without altering the diffusion backbone.
    }
    \label{fig:ctddpms}
    \vspace{-5mm}
\end{figure}

\paragraph{Control Module} 

We introduce a control module that receives a noisy version of the ground truth $\mathbf{x}_0$ as the conditioning input $\mathbf{C}_f$, enabling conditional generation to enhance model learning capability. This approach draws inspiration from classifier-free guidance (CFG)~\citep{classifierfreediffusionguidance,controlnet}, where conditions are injected into the diffusion model to guide data  synthesis. Unlike CFG, we enable fine-grained control via an auxiliary network to inject conditions into the denoising process.

\underline{Noisy input.}
We use Laplace noise to form $\mathbf{C}_{f}$ for its empirical effectiveness (detailed in experimental results): 
{\small
\begin{align}
     \mathbf{C}_{f} = \mathbf{x}_0 + \text{Laplace}(b)\,,
\end{align}
} 
where $\text{Laplace}(b)$ represents the Laplace noise with mean $0$ and variance $2b^2$ -- $b$ is the \emph{noise scale}. This noise-added input serves as a form of $L_2$ regularization during training, aiming to improve generalization with limited data.

The control module and the denoising network interact via \emph{input fusion} and \emph{control fusion}.

\underline{Input fusion.} The control module and the denoising network share the same input $\mathbf{x}_\text{t}$ and $t$, i.e., the sinusoidal timestep embedding~\citep{tabddpm,relddpm,tabsyn}:
{\small
\begin{equation}
     \mathbf{t}_{\text{emb}} = \text{Linear}(\text{SiLU}(\text{Linear}(\text{SinTimeEmb}(t))))\,, 
\end{equation}
}
the timestep embedding is fused with $\mathbf{x}_\text{t}$ for both modules: 
{\small
\begin{align}
& \mathbf{x}_{\text{left}} = \text{Linear}(\mathbf{x}_\text{t}) + \mathbf{t}_{\text{emb}}\,, \nonumber\\
& \mathbf{x}_{\text{right}} = \text{Linear}(\text{zero\_convolution}(\mathbf{C}_{f}) + \mathbf{x_{t}}) + \mathbf{t}_{\text{emb}} \,. 
\end{align}
}

Here, $\mathbf{x}_\text{left}$ is the output of the top linear layer of the denoising network, and $\mathbf{x}_\text{right}$ is that of the top \emph{zero convolution} and linear layers (see Figure~\ref{fig:ctddpms}).
A zero convolution layer, implemented as a $1 \times 1$ convolution with zero-initialized weights and biases, is used to match dimensions and preserve alignment. The purpose is to learn effective encoded information of the condition  before injecting it into the decoding layers of the denoising network to guide the output.

\underline{Control fusion.}
The embedding $\mathbf{x}_\text{right}$ then goes through an encoder block and a mid block, which are copied from the denoising network. The denoising network is pretrained and kept frozen when training the control module, to retain its strong data synthesis capability for denser, low-dimensional data. 
Let $\mathbf{h}_{\text{left}}$ be the hidden outputs of the denoising network, and $\mathbf{h}_{\text{right}}$ be the output of the counterpart blocks in the control module. These outputs are fused as:
{\small
\begin{align}
\mathbf{h}_\text{fusion} = \mathbf{h}_\text{left} + \text{zero\_convolution}(\mathbf{h}_\text{right})\,.
\end{align}
}

Note the red line in Figure~\ref{fig:ctddpms}, which connects the raw control module input directly to the output of the denoising network.
This design preserves low-level condition information and complements the  encoded high-level guidance, improving controllability and generation quality, especially in high-dimensional scenarios.

\begin{figure}[ht]
\centering
\begin{minipage}{0.48\textwidth}
\begin{algorithm}[H]
\caption{Training} \label{alg:training}
\small
\begin{algorithmic}
\FOR{each step}
    \STATE Sample $\mathbf{x}_0 \sim q(\mathbf{x}_0), t$, and $\mathbf{\epsilon}$
    \STATE Compute perturbed data $\mathbf{x}_t$ and $\mathbf{C}_f$
    \STATE Calculate loss $\left\| \mathbf{\epsilon} - \mathbf{\epsilon}_\theta(\mathbf{x}_{t}, \mathbf{C}_{f}, t) \right\|^2$
    \STATE Update the control module via backpropagation
\ENDFOR
\STATE \textbf{return} Trained \model\
\end{algorithmic}
\end{algorithm}
\end{minipage}
\hfill
\begin{minipage}{0.48\textwidth}
\vspace{-2.3ex} 
\begin{algorithm}[H]
\caption{Sampling} \label{alg:sampling}
\small
\begin{algorithmic}
    \STATE Sample $\mathbf{x}_T$, $\mathbf{C}_f$
    \FOR{$t=T, \dotsc, 1$}
      \STATE Predict $\mathbf{\epsilon}_\theta(\mathbf{x}_t, \mathbf{C}_f, t)$ using \model\ output
      \STATE Obtain $\mathbf{x}_{t-1}$ through the reverse equation of \model\
    \ENDFOR
    \STATE \textbf{return} $\mathbf{x}_0$
\end{algorithmic}
\end{algorithm}
\end{minipage}
\end{figure}

\paragraph{Training and Inference}
The training objectives of \model\ with DDPM and score-based SDE are:
{\small
\begin{align}
     &L_{DDPM}(\theta) = \mathbb{E}_{\mathbf{x_0}, t, \mathbf{\epsilon}, \mathbf{C}_{f}}{ \| \mathbf{\epsilon} - \mathbf{\epsilon}_\theta(\mathbf{x}_t, t, \mathbf{C}_{f}) \|^2}\,, \quad \nonumber \\
     &L_{SDE}(\theta) = \mathbb{E}_{\mathbf{x_0}}\mathbb{E}_{\mathbf{x_t}|\mathbf{x_0}}\mathbb{E}_{\mathbf{C}_{f}}{ \| \mathbf{\epsilon} - \mathbf{\epsilon}_\theta(\mathbf{x}_t, t, \mathbf{C}_{f}) \|^2}\,.
\end{align}
}

As noted earlier, during training, gradients flow through the entire model, but only the control module’s parameters are updated while the denoising network remains frozen.

Algorithm~\ref{alg:training} and~\ref{alg:sampling} summarize the training and  inference (i.e., sampling) of \model, respectively.

\section{Theoretical Results} \label{sec:theproof}

We show that adding appropriate noise to the condition leads to a training objective analogous to one with the $L_2$ regularization, which is known to help prevent overfitting to the true data distribution~\citep{l2regularization}. For simplicity, in this section, we use $\mathbf{C}_{f}$ to denote the noise-free condition, and $\mathbf{C}_{f} + \mathbf{\tilde{\epsilon}}$ the noise-added condition.
We present a summary of the theoretical results here and full proofs in Appendix B.
\begin{theoremrestate}\label{thm:1}
Let the training objective of \model\ be to minimize $\mathcal{L} = \mathbb{E}_{\mathbf{x_0}, t, \mathbf{\epsilon}, \mathbf{C}_{f}}{ \| \mathbf{\epsilon} - \mathbf{\epsilon_\theta}(\mathbf{x_t}, t, \mathbf{C}_{f}) \|^2}$, and let the training objective with a noise added condition be to minimize $\tilde{\mathcal{L}} = \mathbb{E}_{\mathbf{x_0}, t, \mathbf{\epsilon}, \mathbf{C}_{f}, \mathbf{\tilde{\epsilon}}}{ \| \mathbf{\epsilon} - \mathbf{\epsilon_\theta}(\mathbf{x_t}, t, \mathbf{C}_{f} + \mathbf{\tilde{\epsilon}}) \|^2}$, where $\tilde{\epsilon}$ is a noise drawn from a distribution with mean $0$ and variance $\eta^2$. 
Then, $\tilde{\mathcal{L}} = \mathcal{L} + \eta^2 \mathcal{L}^R$ holds, 
where $\mathcal{L}^R$ is a regularization term in Tikhonov form.
\end{theoremrestate}

\begin{proof}[Proof Sketch]
\underline{Expansion of the training objectives. } 
    Following prior work~\cite{trainwithnoise,fapproximation}, we define $y(\mathbf{C}_f) := \epsilon_\theta(\mathbf{x}_t, t, \mathbf{C}_{f})$ to simplify the notation. Then,   
    $\mathcal{L} = \mathbb{\mathbb{E}}_{\mathbf{x}_0, t, \epsilon, \mathbf{C}_{f}}{ \| \epsilon - \epsilon_\theta(\mathbf{x}_t, t, \mathbf{C}_{f}) \|^2}$
    can be expanded as:
    {\small
\begin{align}
      \mathcal{L} &= \int\int\int\int \{ y(\mathbf{C}_{f}) - \epsilon \}^2 p(\epsilon|t,\mathbf{x}_0, \mathbf{C}_{f})p(t,\mathbf{x}_0, \mathbf{C}_{f})\, dt\, d\mathbf{x}_0 \nonumber \\
      &\, d\epsilon\, d\mathbf{C}_{f}\,. \label{eq:E1}
\end{align}
}
Also, $\tilde{\mathcal{L}} = \mathbb{E}_{\mathbf{x}_0, t, \epsilon, \mathbf{C}_{f}, \tilde{\epsilon}}{ \| \epsilon - \epsilon_\theta(\mathbf{x}_t, t, \mathbf{C}_{f} + \tilde{\epsilon}) \|^2}$ becomes:
{\small
\begin{align}
      \tilde{\mathcal{L}} &= \int\int\int\int\int\{ y(\mathbf{C}_{f}+\tilde{\epsilon}) - \epsilon \}^2p(\epsilon|t,\mathbf{x}_0, \mathbf{C}_{f})p(t,\mathbf{x}_0, \mathbf{C}_{f}) \nonumber \\
      &p(\tilde{\epsilon})\, dt\, d\mathbf{x}_0\, d\epsilon\, d\mathbf{C}_{f}\, d\tilde{\epsilon}. \label{eq:Etilde1}
\end{align}}

\noindent\underline{Some preliminaries. }
We expand $y(\mathbf{C}_{f}+\tilde{\epsilon})$ using a Taylor series, assuming that the noise amplitude is small such that terms of order $O(\tilde{\epsilon}^3)$ and above can be omitted: 
{\small
\begin{align}
      y(\mathbf{C}_{f}+\tilde{\epsilon}) &= y(\mathbf{C}_{f}) + \sum_i\tilde{\epsilon_i}\left. \frac{\partial y}{\partial \mathbf{C}_{f,i}} \right|_{\tilde{\epsilon}=0} + \frac{1}{2}\sum_i\sum_j\tilde{\epsilon_i}\tilde{\epsilon_j} \nonumber \\
      &\left. \frac{\partial^2 y}{\partial \mathbf{C}_{f,i}\partial \mathbf{C}_{f,j}} \right|_{\tilde{\epsilon}=0}+O(\tilde{\epsilon}^3) \label{eq:tayerseries1}
\end{align}}
This small-noise assumption is both theoretically and practically motivated. Practically, large-magnitude noise may overwhelm the original data features and hinder model training by masking meaningful patterns, as evidenced in Figure \ref{fig:sub1}, where the model's performance drops when the noise scale increases. Theoretically, assuming small noise levels facilitates tractable analysis --- a common practice in classical regularization literature~\cite{trainwithnoise}.

By the fact that $\tilde{\epsilon}$ is chosen from a distribution with mean $0$ and variance $\eta^2$, we obtain:
\begin{equation}
\small
\int \tilde{\epsilon}_i p(\tilde{\epsilon})\,d\tilde{\epsilon} = 0, 
\int \tilde{\epsilon}_i \tilde{\epsilon}_j p(\tilde{\epsilon})\,d\tilde{\epsilon} = \eta^2 \delta_{ij},\ 
\delta_{ij} =
\begin{cases}
1 & \text{if } i = j \\
0 & \text{if } i \neq j
\end{cases}.
\label{eq:eplisontilde1}
\end{equation}
\noindent\underline{The training objective with noise is a regularization. }
By substituting Eq.~\eqref{eq:tayerseries1} into Eq.~\eqref{eq:Etilde1}, we have:
{\small
\begin{align}
    \tilde{\mathcal{L}} &= \mathcal{L} + \eta^2 \mathcal{L}^R\,, \text{where} \label{eq:16}
    \end{align}
    \vspace{-5mm}
    \begin{align}
      \mathcal{L}^R &= \int\int\int\int\sum_i \{ (\frac{\partial y}{\partial \mathbf{C}_{f,i}})^2 + (y(\mathbf{C}_{f})-\epsilon)\frac{\partial^2 y}{\partial \mathbf{C}_{f,i}^2} \} \nonumber \\
      &p(\epsilon|t,\mathbf{x}_0, \mathbf{C}_{f})p(t,\mathbf{x}_0, \mathbf{C}_{f})\, dt\, d\mathbf{x}_0\, d\epsilon\, d\mathbf{C}_{f}. \label{eq:et71}
\end{align}
}
\noindent\underline{The training objective with noise resembles  $L_2$ regularization.}
We define two terms:
{\small
\begin{align}
&\langle\epsilon \mid t, \mathbf{x}_0, \mathbf{C}_{f}\rangle = \int \epsilon\, p(\epsilon \mid t, \mathbf{x}_0, \mathbf{C}_{f})\, d\epsilon, \quad \nonumber \\
&\langle\epsilon^2 \mid t, \mathbf{x}_0, \mathbf{C}_{f}\rangle = \int \epsilon^2\, p(\epsilon \mid t, \mathbf{x}_0, \mathbf{C}_{f})\, d\epsilon.
\label{eq:eps_expectation}
\end{align}}
Then, Eq.~\eqref{eq:E1} is expanded as follows:
{\small
\begin{align}
    & \mathcal{L} = \int\int\int\int \{ y(\mathbf{C}_{f}) - \langle\epsilon | t, \mathbf{x}_0, \mathbf{C}_f\rangle \}^2 p(\epsilon|t,\mathbf{x}_0, \mathbf{C}_f) \nonumber \\
    &p(t,\mathbf{x}_0, \mathbf{C}_{f})\, dt\, d\mathbf{x}_0\, d\epsilon\, d\mathbf{C}_{f} + \int\int\int\int\{ \langle\epsilon^2 | t, \mathbf{x}_0, \mathbf{C}_f\rangle \nonumber \\
    &-\langle\epsilon | t, \mathbf{x}_0, \mathbf{C}_{f}\rangle^2 \}p(\epsilon|t,\mathbf{x}_0, \mathbf{C}_{f})p(t,\mathbf{x}_0, \mathbf{C}_{f})\, dt\, d\mathbf{x}_0\, d\epsilon\, d\mathbf{C}_{f}.\label{eq:19}
\end{align}}

Observe that in Eq.~\eqref{eq:19}, only the first integral term involves the parameters of a neural network (i.e.,  $y(\mathbf{C}_f) =\epsilon_\theta(\mathbf{x}_t, t, \mathbf{C}_{f})$, the denoising network to be trained), while the second term only depends on the ground-truth noise.
Therefore, $\mathcal{L}$ is minimized when 
$y(\mathbf{C}_f) = \langle\epsilon | t, \mathbf{x}_0, \mathbf{C}_f\rangle$.

Also, recall that $\tilde{\mathcal{L}} = \mathcal{L} + \eta^2 \mathcal{L}^R$ by Eq.~\eqref{eq:16}.
Consider the second term of $\mathcal{L}^R$ in Eq.~\eqref{eq:et71} and denote it as $\mathcal{L}_2^R$. Then, $\mathcal{L}_2^R$ can be rewritten as follows:
{\small
\begin{align}
    \mathcal{L}^R_{2} &= \int\int\int\int\sum_i \{ y(\mathbf{C}_f) - \langle\epsilon | t, \mathbf{x}_0, \mathbf{C}_f\rangle \} \frac{\partial^2 y}{\partial \mathbf{C}_{f,i}^2} \nonumber \\
    &p(\epsilon|t,\mathbf{x}_0, \mathbf{C}_f)p(t,\mathbf{x}_0, \mathbf{C}_f)\, dt\, d\mathbf{x}_0\, d\epsilon\, d\mathbf{C}_f.
\end{align}}
Thus, when $\mathcal{L}$ is minimized at $y(\mathbf{C}_f) = \langle\epsilon | t, \mathbf{x}_0, \mathbf{C}_f\rangle$, $\mathcal{L}_2^R = 0$. 
In this case, $\mathcal{L}^R$ can be rewritten as:
{\small
\begin{align}
    \mathcal{L}^R & = \int\int\int\sum_i (\frac{\partial y}{\partial \mathbf{C}_{f,i}})^2 p(t,\mathbf{x}_0, \mathbf{C}_f)\, dt\, d\mathbf{x}_0\, d\mathbf{C}_f\,, \label{eq:er}
\end{align}}
which corresponds to the Tikhonov form. 
\end{proof}

\section{Experiments}

\paragraph{Datasets} We use real-world datasets: \textbf{GE}, \textbf{CL}, \textbf{MA}, \textbf{ED}, \textbf{UN}, \textbf{UG}, and \textbf{EG}. These datasets, as summarized in Table~\ref{tab:datasets} and detailed in Appendix C, have been chosen for their large number of feature dimensions (up to 241) relative to the number of rows (a few thousands).

\paragraph{Competitors} We compare \model\ with six baseline models including state-of-the-art (SOTA) diffusion-based tabular data synthesis models: \textbf{SMOTE}~\citep{smote}, \textbf{TVAE}~\citep{ctgan}, \textbf{CTGAN}~\citep{ctgan}, \textbf{TabDDPM}~\citep{tabddpm}, \textbf{RelDDPM}~\citep{relddpm}, and \textbf{TabSyn}~\citep{tabsyn} (SOTA). These models are detailed in Appendix C, together with implementation details of these models and our \model\ model.

\begin{table*}[ht] 
\setlength{\tabcolsep}{4pt}
    \centering 
    {\small
	\begin{tabular}{lccccc|r||cc|r}
            \toprule[1pt]
            \textbf{Method} &  \textbf{GE} & \textbf{CL} & \textbf{MA} & \textbf{ED} & \textbf{UN} & \textbf{Avg. Gap} & \textbf{UG} & \textbf{EG} & \textbf{Avg. Gap}\\
            \midrule 
             & AUC $\uparrow$ & AUC $\uparrow$ & AUC $\uparrow$ & AUC $\uparrow$ & AUC $\uparrow$ & $\% \downarrow$ & RMSE $\downarrow$ & RMSE $\downarrow$ & $\% \downarrow$\\
                                    \midrule
                                    Real    &   0.896 & 1 & 0.999 & 0.990 & 0.986 & $0\%$ & 0.036 & 0.115 & $0\%$\\ \midrule
                                    SMOTE    &  \underline{0.865} & \underline{0.999} & \underline{0.996} & \textbf{0.990} & \underline{0.976} & \underline{$0.98\%$} & \underline{0.132} & \underline{0.118} & \underline{$134.64\%$} \\
                                    TVAE      &  - & - & - & 0.825 & 0.819 & $16.80\%$ & 0.775& 0.219 & $1071.61\%$\\
                                    CTGAN     &  0.133 & 0.653 & 0.765 & 0.383 & 0.534 & $50.09\%$ & 0.754& 0.266 & $1062.87\%$\\
                                    TabDDPM   &  0.504 & 0.550 & 0.805 & 0.459 & 0.618 & $39.83\%$ & 2.216&0.346 & $3128.21\%$\\
                                    RelDDPM   &  0.839 & 0.665 & 0.687 & 0.981 & 0.930 & $15.54\%$ & 0.463& 0.122 & $596.10\%$\\
                                    TabSyn    &  0.791 & 0.984 & \underline{0.996} & 0.952 & 0.864 & $5.97\%$& 0.267& 0.166 & $343.01\%$\\ \midrule
                                    
                                    \textbf{\model\ (Ours)} & \textbf{0.890}\textsuperscript{*} & \textbf{1}\textsuperscript{*} & \textbf{0.999}\textsuperscript{*} & \underline{0.988} & \textbf{0.979}\textsuperscript{*} & $\textbf{0.32\%}$& \textbf{0.069}\textsuperscript{*} & \textbf{0.114}\textsuperscript{*} & $\textbf{45.83\%}$\\
            \toprule[1.0pt]
            & F1 $\uparrow$ & F1 $\uparrow$ & F1 $\uparrow$ & F1 $\uparrow$ & F1 $\uparrow$ & $\% \downarrow$  & R2 $\uparrow$ & R2 $\uparrow$ & $\% \downarrow$\\ 
                                    \midrule
                                     Real   &  0.637 & 0.991 & 0.993 & 0.928 & 0.923 & $0\%$& 0.998 & 0.646 &$0\%$\\ \midrule
                                    SMOTE    & \underline{0.598} & \underline{0.968} & \underline{0.991} & \textbf{0.919} & \underline{0.889} & \underline{$2.66\%$}& \underline{0.970} & \underline{0.628} & \underline{$2.80\%$}\\
                                    TVAE     & - & - & - & 0.673 & 0.604 & $31.02\%$ & $<0$ & $<0$ & -\\
                                    CTGAN    & 0.465 & 0.190 & 0.884 & 0.404 & 0.580 & $42.49\%$& 0.031 & $<0$ & $96.89\%$\\
                                    TabDDPM  & 0.092 & 0.018 & 0.872 & 0.546 & 0.620 & $53.98\%$& $<0$ & $<0$ & -\\
                                    RelDDPM  & 0.554 & 0.334 & 0.871 & 0.905 & 0.841 & $20.60\%$& 0.634& 0.603 & $21.56\%$\\
                                    TabSyn   & 0.450 & 0.845 & 0.983 & 0.849 & 0.711 & $15.32\%$& 0.879& 0.270 & $35.06\%$\\ \midrule
                                    
                                    \textbf{\model\ (Ours)} & \textbf{0.635}\textsuperscript{*} & \textbf{0.983}\textsuperscript{*} & \textbf{0.994}\textsuperscript{*} & \underline{0.918} & \textbf{0.898}\textsuperscript{*} & $\textbf{0.98\%}$ & \textbf{0.991}\textsuperscript{*} & \textbf{0.657}\textsuperscript{*} & $\textbf{0.35\%}$ \\
		\bottomrule[1.0pt] 
		\end{tabular}
        }
        \caption{ Machine learning test results: The GE, CL, MA, ED, and UN datasets are used for classification, with results in AUC and F1. The UG and EG datasets are for regression, with results in RMSE and R2. Symbol `-' indicates cases where the generative model collapsed, resulting in only a single class of generated samples, or negative R2 values such that the average gap becomes invalid. The best results are in boldface, while the second best are underlined. Symbol `\textsuperscript{*}' denotes values where \model\ significantly outperforms \textit{both} top baselines SMOTE and TabSyn ($p < 0.05$ in t-tests).}

\label{tab:mlperformance}
\end{table*}

\subsection{Results}\label{sec:results}

\paragraph{Machine Learning Tests}\label{sec:mltests}
Following existing work~\citep{stasy,codi,tabsyn}, we test the effectiveness of \model\ mainly through machine learning tasks: (1) We train each model with the training set of each dataset. (2) Once trained, we use each model to synthesize a dataset of the same size of the respective training set. (3) We use the original training set and the synthesized set  to separately train a downstream classifier or regression model (XGBoost and XGBoostRegressor, ``\textbf{downstream model}'' hereafter). (4) We evaluate these trained downstream models on the test set. 

Following prior work~\citep{tabsyn}, we report classification results in the Area Under the ROC Curve (\textbf{AUC}) and \textbf{F1}, and regression results using the Root Mean Squared Error (\textbf{RMSE}) and R-Squared  (\textbf{\boldmath $R2$}).

Table~\ref{tab:mlperformance} summarizes the results, where ``Real'' denotes the downstream model trained on the original data; each model (e.g., \model) denotes the performance of the downstream model trained on data synthesized by that model. 
``Avg.~Gap'' denotes the (relative) performance gap between a model and ``Real'', averaged across the same type (classification or regression) of datasets. A smaller gap means a better model that better fits the original distribution. 

Our model \model\ reports the best results in almost all cases, except on ED, where SMOTE is marginally higher in AUC. The average gap in AUC, F1, RMSE, and R2 of \model\ are $67.35\%$, $63.16\%$, $65.96\%$, and $87.50\%$ lower than those of the best baseline model (SMOTE), respectively. Comparing with the SOTA model TabSyn, the gains are even higher, i.e., $94.64\%$, $93.60\%$, $86.64\%$, and $99.00\%$.

These results demonstrate the challenges for  existing diffusion-based models to learn from sparse, high-dimensional  data, and the effectiveness of \model\ in addressing such challenges. SMOTE, while being an early model, shows high effectiveness, because it generates new samples through adding noise (directly) to existing samples which resembles our idea. TVAE and CTGAN are based on VAE and GAN, which also suffer from the data sparsity issue.

\paragraph{Ablation Study}\label{sec:ablation} We further show the importance of the control module of \model\ by comparing \model\ with six alternative variants of TabSyn: 
(1)~\textbf{Train$\times2$} doubles the number of training epochs for TabSyn. 
(2)~\textbf{Data$\times2$} duplicates each training sample with a Laplace noise of scale $0.01$. 
(3)~\textbf{Model$\times2$} duplicates each training sample as above and doubles the number of model parameters by expanding the hidden layers of TabSyn to match that of \model\ (\model\ is $1.8 \times$ the size of TabSyn).
(4)~\textbf{NoiseCond} applies the same noisy $\mathbf{x}_0$ signal (used in \model's control module) as an additional input to TabSyn, without structural changes.
(5)~\textbf{Dropout-Reg} applies Dropout with a rate of 0.1 to the hidden layers of TabSyn during training, serving as a standard regularization baseline.
(6)~\textbf{JointTrain} trains the control module and the diffusion model jointly in a single stage. This setting evaluates whether our staged training strategy contributes to performance improvements by better stabilizing the learning of the control signal.
We also compare with \textbf{w/o-lastfusion}, i.e., \model\ without the last $h_{fusion}$ connection (the red line in Figure~\ref{fig:ctddpms}) to the denoising network, to verify the importance of this fusion connection.

\begin{table}[ht]
\centering
\small

\setlength{\tabcolsep}{1.2pt}
\begin{tabular}{lcccccccc}
\toprule
\textbf{Method} & \textbf{GE} & \textbf{CL} & \textbf{MA} & \textbf{ED} & \textbf{UN} & \textbf{UG} & \textbf{EG} & \textbf{Average} \\
\midrule
\textbf{Metric} & F1$\uparrow$ & F1$\uparrow$ & F1$\uparrow$ & F1$\uparrow$ & F1$\uparrow$ & R2$\uparrow$ & R2$\uparrow$ & - \\
\midrule
Train$\times2$ & 0.443 & 0.831 & 0.988 & 0.852 & 0.698 & 0.860 & 0.390 & 0.746 \\
Data$\times2$ & 0.463 & 0.871 & 0.987 & 0.865 & 0.820 & 0.937 & 0.458 & 0.801 \\
Model$\times2$ & 0.522 & 0.872 & 0.988 & 0.892 & 0.816 & 0.884 & 0.569 & 0.822 \\
NoiseCond & 0.427 & 0.778 & 0.985 & 0.878 & 0.790 & 0.872 & 0.530 & 0.751 \\
Dropout-Reg & 0.395 & 0.794 & 0.989 & 0.835 & 0.697 & 0.812 & 0.429 & 0.707 \\
JointTrain & 0.525 & 0.970 & 0.977 & 0.857 & 0.732 & 0.766 & 0.569 & 0.772 \\
\shortstack[l]{w/o-lastfusion} & 0.619 & 0.982 & 0.992 & 0.909 & 0.883 & 0.805 & 0.655 & 0.835 \\
\textbf{\model} & \textbf{0.635} & \textbf{0.983} & \textbf{0.994} & \textbf{0.918} & \textbf{0.898} & \textbf{0.991} & \textbf{0.657} & \textbf{0.868} \\
\bottomrule
\end{tabular}
\caption{Ablation study results.}

\label{tab:differentmodules}
\end{table}

We repeat the machine learning tests as above and report the results in F1 and R2 in Table~\ref{tab:differentmodules}. Results in AUC and RMSE share similar patterns and are in Appendix D. \model\ outperforms all these variants consistently, confirming the importance of the control module, our staged training strategy, and the last fusion connection. 

\underline{Impact of Noise Type.} We also replace Laplace noise in our control module with Gaussian and uniform noise. 
We find that while using Laplace noise typically leads to the best accuracy, the gain is often not too large. This confirms the robustness of \model\ and the effectiveness of using Laplace noise. Detailed results are provided in Appendix~D.

\paragraph{Parameter Study} Figure~\ref{fig:sub1} shows the downstream model accuracy (F1 and R2) when the Laplace noise scale in \model\ is varied from $0$ to $1000$. 
As the noise scale increases, model accuracy decreases at start due to the impact of regularization by the noise scale (variance), which follows our theoretical results. When the noise grows further, the control module gradually fails, the regularization becomes ineffective, and the curves rise back until 
converging to the performance of the original diffusion model (i.e., \model\ without the control module, ``NC'' in Figure~\ref{fig:sub1}).

Figure~\ref{fig:sub1} also shows results where the noise scale is 0, which means to train \model\ without additive noise on the control signal $\mathbf{C}_f$. We see that model performance drops comparing with using small Laplace noise as noted above. This supports our theoretical finding that noise injection implicitly imposes $L_2$ regularization, promoting smoother mappings and better generalization.

\begin{figure}[t]
    \centering
    \subfloat[Impact of noise scale\label{fig:sub1}]{
        \includegraphics[width=0.45\linewidth]{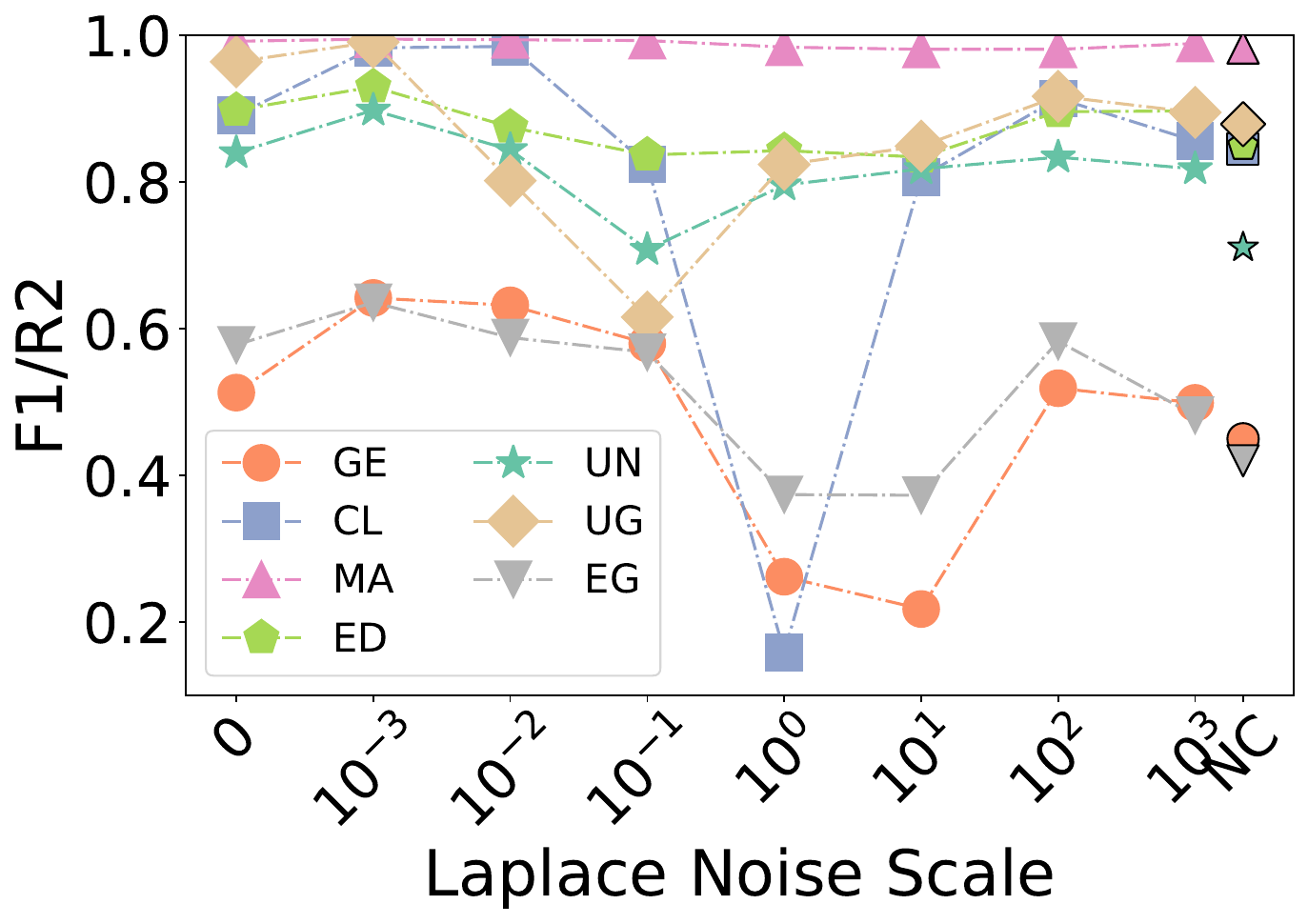}
    }
    \hfill
    \subfloat[\model\ on RelDDPM\label{fig:sub2}]{
        \includegraphics[width=0.45\linewidth]{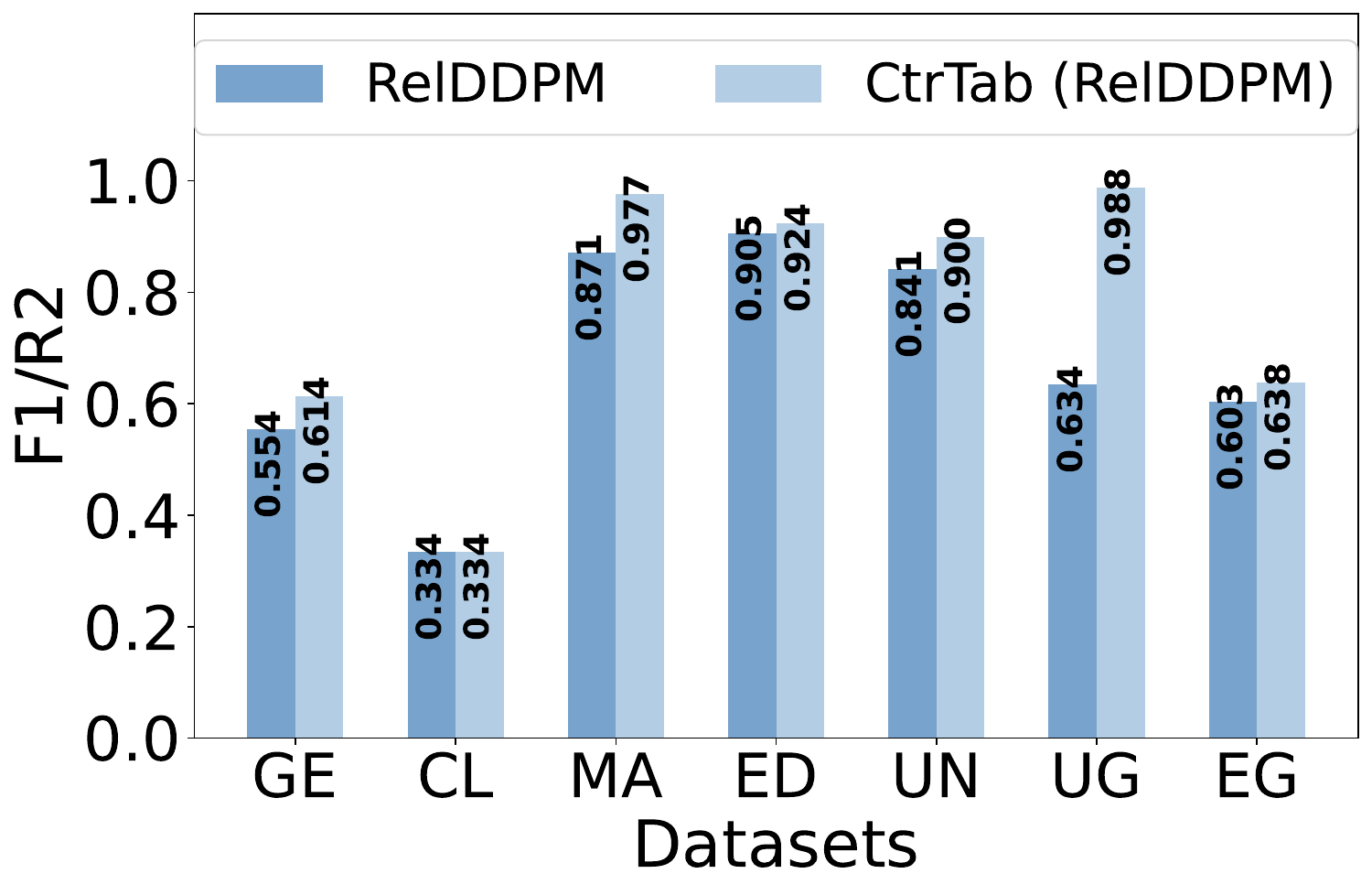}
    }
    \caption{Impact of noise scale and the control module.}
    \label{fig:combined}
\end{figure}

\paragraph{Control Module Applicability}
Our control module is not tied to TabSyn. We further integrate the control module with RelDDPM~\citep{relddpm}, a representative DDPM-based model. 
RelDDPM uses the standard diffusion forward and reverse processes (Eqs.~\ref{eq:ddpmforward} and~\ref{eq:ddpmbackward}) with a classifier-guided conditional generation.
In our high-dimensional setting, the classifier is not required. We remove it and only adopt RelDDPM's denoising network to replace the denoising module in \model.
This effectively substitutes the SDE-based diffusion backbone with a DDPM-style architecture, while keeping the control module intact. 
The accuracy (F1 and R2) results are shown in Figure~\ref{fig:sub2}, where \model\ (with RelDDPM) also outperforms RelDDPM consistently. This verifies the applicability of our control module. 

\paragraph{Case Study} We explore an extremely high-dimensional setting uncommon in the literature, to assess model robustness. We test \model\ on \textbf{ST}~\cite{stadyn} and \textbf{AC}~\cite{arcene} with 1,084 and 10,001 dimensions (4,998 and 100 rows, see Appendix~C), respectively, against the two diffusion baselines. 
Figure~\ref{fig:hbar} shows downstream classification performance, confirming the robustness of \model\ in this extreme setting. 

On AC, the classifier trained on real data (``Real-AC'') performs poorly in F1, due to the presence of 3,000 noisy dimensions. In contrast, the synthetic samples, especially those generated by \model, yield much higher F1 scores. 
This improvement is not only because the generative process captures more generalizable patterns and tends to suppress noisy or spurious correlations, but also because \model\  is trained with implicit regularization, which encourages the generator to focus on informative structures in the data. 

Real-AC in AUC is not as bad. AUC measures a model’s capability to rank positive and negative samples differently, while F1 is impacted by the classification threshold. 
As a result, AUC is less sensitive to noise and threshold selection.
\begin{figure}[ht]
    \centering
    \includegraphics[width=0.8\linewidth]{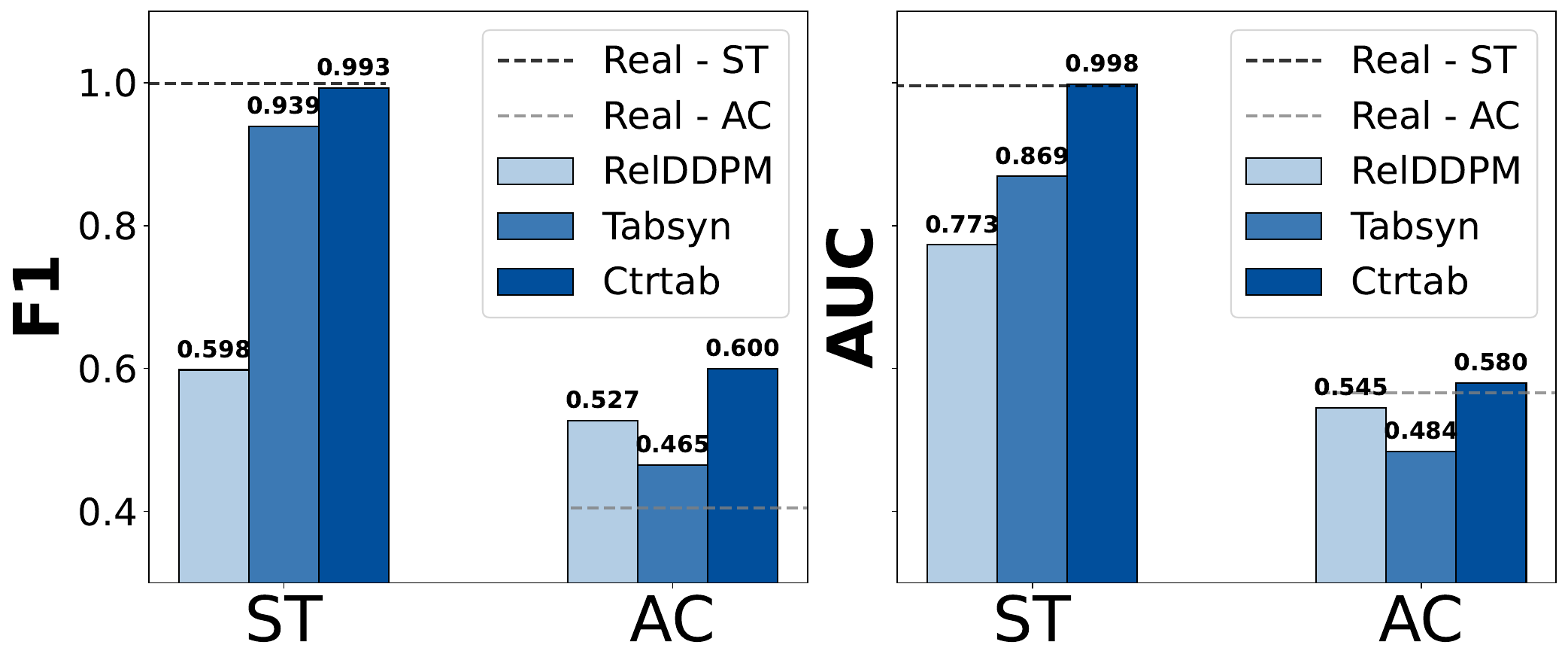}
    \caption{Case study on real-world extremely high-dimensional dataset. The dashed lines correspond to models trained on original data.}
    \label{fig:hbar}
    \vspace{-5mm}
\end{figure}

\paragraph{Additional Results} In the appendix, we further include results on model training and inference times, model performance with varying percentages of training data and on non-high-dimensional (i.e., conventional) datasets, distribution visualization results, results on distance to closest records and the impact of using different types of noise. \label{sec:experiments}

\section{Conclusion}

We proposed \model, a model to enhance the fitting capability of diffusion generative models towards high-dimensional  tabular data with limited samples. Through an explicit noise-conditioned control and training with a method similar to $L_2$ regularization, \model\ synthesizes high-quality tabular data as evidenced by experiments over real datasets. The results show that machine learning models trained with data synthesized by \model\ are much more accurate than those trained with data synthesized by existing models including the SOTA, with an accuracy gain of over 90\% on average.

We extended tabular data synthesis to tables with up to 10,001 dimensions. Future work could scale to even higher dimensions. This work focuses on advancing the quality and controllability of  tabular data synthesis in settings where privacy is not a constraint. Exploring how our methods interact with privacy constraints is left as future work.

\bibliography{main}

\begin{thebibliography}{49}
\providecommand{\natexlab}[1]{#1}

\bibitem[{Aggarwal and Yu(2004)}]{condensation}
Aggarwal, C.~C.; and Yu, P.~S. 2004.
\newblock A Condensation Approach to Privacy Preserving Data Mining.
\newblock In \emph{EDBT}, 183--199.

\bibitem[{An et~al.(2025)An, Woo, Lim, Kim, Hong, and Jeon}]{MaskedLanguageModeling}
An, S.; Woo, G.; Lim, J.; Kim, C.; Hong, S.; and Jeon, J.-J. 2025.
\newblock Masked Language Modeling Becomes Conditional Density Estimation for Tabular Data Synthesis.
\newblock In \emph{AAAI}, 15356--15364.

\bibitem[{Arjovsky, Chintala, and Bottou(2017)}]{wgan}
Arjovsky, M.; Chintala, S.; and Bottou, L. 2017.
\newblock {W}asserstein Generative Adversarial Networks.
\newblock In \emph{ICML}, 214--223.

\bibitem[{Barak et~al.(2007)Barak, Chaudhuri, Dwork, Kale, McSherry, and Talwar}]{Fourierdecompositions}
Barak, B.; Chaudhuri, K.; Dwork, C.; Kale, S.; McSherry, F.; and Talwar, K. 2007.
\newblock Privacy, Accuracy, and Consistency too.
\newblock In \emph{PODS}, 273--282.

\bibitem[{Bishop(1995)}]{trainwithnoise}
Bishop, C.~M. 1995.
\newblock Training with Noise is Equivalent to {Tikhonov} Regularization.
\newblock \emph{Neural Computation}, 7(1): 108--116.

\bibitem[{Bishop(2006)}]{patternrecognition}
Bishop, C.~M. 2006.
\newblock \emph{{Pattern Recognition and Machine Learning (Information Science and Statistics)}}.
\newblock Berlin, Heidelberg: Springer-Verlag.

\bibitem[{Borah and Bhattacharyya(2020)}]{malware}
Borah, P.; and Bhattacharyya, D.~K. 2020.
\newblock {TUANDROMD (Tezpur University Android Malware Dataset)}.
\newblock UCI Machine Learning Repository.
\newblock \url{https://doi.org/10.24432/C5560H}.

\bibitem[{Chapman and Jain(1994)}]{musk2}
Chapman, D.; and Jain, A. 1994.
\newblock {Musk (Version 2)}.
\newblock UCI Machine Learning Repository.
\newblock \url{https://doi.org/10.24432/C51608}.

\bibitem[{Chawla et~al.(2002)Chawla, Bowyer, Hall, and Kegelmeyer}]{smote}
Chawla, N.~V.; Bowyer, K.~W.; Hall, L.~O.; and Kegelmeyer, W.~P. 2002.
\newblock {SMOTE}: Synthetic Minority Over-sampling Technique.
\newblock \emph{Journal of Artificial Intelligence Research}, 16: 321--357.

\bibitem[{Chen et~al.(2019)Chen, Jajodia, Liu, Park, Sokolov, and Subrahmanian}]{itsgan}
Chen, H.; Jajodia, S.; Liu, J.; Park, N.; Sokolov, V.; and Subrahmanian, V.~S. 2019.
\newblock {FakeTables}: Using {GANs} to Generate Functional Dependency Preserving Tables with Bounded Real Data.
\newblock In \emph{IJCAI}, 2074--2080.

\bibitem[{D’souza, M, and Sarawagi(2025)}]{imbalancedclassification}
D’souza, A.; M, S.; and Sarawagi, S. 2025.
\newblock Synthetic Tabular Data Generation for Imbalanced Classification: The Surprising Effectiveness of an Overlap Class.
\newblock In \emph{AAAI}, 16127--16134.

\bibitem[{Elman et~al.(2020)Elman, Minnier, Chang, and Choi}]{noiseaccumulation}
Elman, M.~R.; Minnier, J.; Chang, X.; and Choi, D. 2020.
\newblock Noise Accumulation in High Dimensional Classification and Total Signal Index.
\newblock \emph{Journal of Machine Learning Research}, 21(36): 1--23.

\bibitem[{Fang et~al.(2024)Fang, Liu, Zhang, Zou, Zhang, and Yu}]{resrag}
Fang, L.; Liu, A.; Zhang, H.; Zou, H.~P.; Zhang, W.; and Yu, P.~S. 2024.
\newblock {RES}-{RAG}: Residual-aware {RAG} for Realistic Tabular Data Generation.
\newblock In \emph{NeurIPS 2024 Table Representation Learning Workshop}.

\bibitem[{Guyon et~al.(2004)Guyon, Gunn, Ben-Hur, and Dror}]{arcene}
Guyon, I.; Gunn, S.; Ben-Hur, A.; and Dror, G. 2004.
\newblock {Arcene}.
\newblock UCI Machine Learning Repository.

\bibitem[{Ho, Jain, and Abbeel(2020)}]{DDPM}
Ho, J.; Jain, A.; and Abbeel, P. 2020.
\newblock Denoising Diffusion Probabilistic Models.
\newblock In \emph{NeurIPS}, 6840--6851.

\bibitem[{Ho and Salimans(2021)}]{classifierfreediffusionguidance}
Ho, J.; and Salimans, T. 2021.
\newblock Classifier-Free Diffusion Guidance.
\newblock In \emph{NeurIPS 2021 Workshop on Deep Generative Models and Downstream Applications}.

\bibitem[{Hsieh et~al.(2025)Hsieh, Moreira, Nobre, Sousa, Ouyang, Brereton, Jorge, and Nascimento}]{dallm}
Hsieh, C.; Moreira, C.; Nobre, I.~B.; Sousa, S.~C.; Ouyang, C.; Brereton, M.; Jorge, J.; and Nascimento, J.~C. 2025.
\newblock DALL-M: Context-aware Clinical Data Augmentation with Large Language Models.
\newblock \emph{Computers in Biology and Medicine}, 190: 110022.

\bibitem[{Kim et~al.(2021)Kim, Jeon, Lee, Hyeong, and Park}]{octgan}
Kim, J.; Jeon, J.; Lee, J.; Hyeong, J.; and Park, N. 2021.
\newblock {OCT-GAN}: Neural {ODE}-based Conditional Tabular {GANs}.
\newblock In \emph{WWW}, 1506--1515.

\bibitem[{Kim, Lee, and Park(2023)}]{stasy}
Kim, J.; Lee, C.; and Park, N. 2023.
\newblock {STaSy}: Score-based Tabular Data Synthesis.
\newblock In \emph{ICLR}.

\bibitem[{Kim, Lee, and Park(2025)}]{TabularDataImputation}
Kim, J.; Lee, K.; and Park, T. 2025.
\newblock To Predict or Not to Predict? Proportionally Masked Autoencoders for Tabular Data Imputation.
\newblock In \emph{AAAI}, 17886--17894.

\bibitem[{Kotelnikov et~al.(2023)Kotelnikov, Baranchuk, Rubachev, and Babenko}]{tabddpm}
Kotelnikov, A.; Baranchuk, D.; Rubachev, I.; and Babenko, A. 2023.
\newblock {TabDDPM}: Modelling Tabular Data with Diffusion Models.
\newblock In \emph{ICML}, 17564--17579.

\bibitem[{Lee, Kim, and Park(2023)}]{codi}
Lee, C.; Kim, J.; and Park, N. 2023.
\newblock {CoDi}: Co-evolving Contrastive Diffusion Models for Mixed-type Tabular Synthesis.
\newblock In \emph{ICML}, 18940--18956.

\bibitem[{Li et~al.(2014)Li, Xiong, Zhang, and Jiang}]{DPSynthesizer}
Li, H.; Xiong, L.; Zhang, L.; and Jiang, X. 2014.
\newblock {DPSynthesizer}: Differentially Private Data Synthesizer for Privacy Preserving Data Sharing.
\newblock \emph{Proceedings of the VLDB Endowment}, 7(13): 1677–1680.

\bibitem[{Liu et~al.(2024)Liu, Fan, Tang, Li, and Du}]{relddpm}
Liu, T.; Fan, J.; Tang, N.; Li, G.; and Du, X. 2024.
\newblock Controllable Tabular Data Synthesis Using Diffusion Models.
\newblock \emph{Proceedings of the ACM on Management of Data}, 2(1): 28:1--28:29.

\bibitem[{Loshchilov and Hutter(2019)}]{adamw}
Loshchilov, I.; and Hutter, F. 2019.
\newblock Decoupled Weight Decay Regularization.
\newblock In \emph{ICLR}.

\bibitem[{Lu et~al.(2023)Lu, Chen, Zhang, Shen, Wang, Wang, van Rechem, Fu, and Wei}]{survey2}
Lu, Y.; Chen, L.; Zhang, Y.; Shen, M.; Wang, H.; Wang, X.; van Rechem, C.; Fu, T.; and Wei, W. 2023.
\newblock Machine Learning for Synthetic Data Generation: A Review.
\newblock arXiv:2302.04062.

\bibitem[{Madeo, Lima, and Peres(2013)}]{gesture}
Madeo, R. C.~B.; Lima, C. A.~M.; and Peres, S.~M. 2013.
\newblock Gesture Unit Segmentation Using Support Vector Machines: {Segmenting} Gestures from Rest Positions.
\newblock In \emph{ACM Symposium on Applied Computing}, 46--52.

\bibitem[{Ng(2004)}]{l2regularization}
Ng, A.~Y. 2004.
\newblock {Feature Selection}, {L1} vs. {L2} Regularization, and Rotational Invariance.
\newblock In \emph{ICML}, 78.

\bibitem[{Pang et~al.(2024)Pang, Shafieinejad, Liu, Hazlewood, and He}]{clavaddpm}
Pang, W.; Shafieinejad, M.; Liu, L.; Hazlewood, S.; and He, X. 2024.
\newblock ClavaDDPM: Multi-relational Data Synthesis with Cluster-guided Diffusion Models.
\newblock In \emph{NeurIPS}, 83521--83547.

\bibitem[{Park et~al.(2018)Park, Mohammadi, Gorde, Jajodia, Park, and Kim}]{tablegan}
Park, N.; Mohammadi, M.; Gorde, K.; Jajodia, S.; Park, H.; and Kim, Y. 2018.
\newblock Data Synthesis Based on Generative Adversarial Networks.
\newblock \emph{Proceedings of the VLDB Endowment}, 11(10): 1071--1083.

\bibitem[{Park, Ghosh, and Shankar(2013)}]{gibbssamplers}
Park, Y.; Ghosh, J.; and Shankar, M. 2013.
\newblock Perturbed {Gibbs} Samplers for Generating Large-scale Privacy-safe Synthetic Health Data.
\newblock In \emph{IEEE International Conference on Healthcare Informatics}, 493--498.

\bibitem[{Sanghi and Haritsa(2023)}]{DBMSsynthesis}
Sanghi, A.; and Haritsa, J.~R. 2023.
\newblock Synthetic Data Generation for Enterprise DBMS.
\newblock In \emph{ICDE}, 3585--3588.

\bibitem[{Schreyer et~al.(2024)Schreyer, Sattarov, Sim, and Wu}]{Imb-FinDiff}
Schreyer, M.; Sattarov, T.; Sim, A.; and Wu, K. 2024.
\newblock Imb-FinDiff: Conditional Diffusion Models for Class Imbalance Synthesis of Financial Tabular Data.
\newblock In \emph{ACM International Conference on AI in Finance}, 617--625.

\bibitem[{Shi et~al.(2024)Shi, Hua, Xu, Zhang, Ermon, and Leskovec}]{tabdiff}
Shi, J.; Hua, H.; Xu, M.; Zhang, H.; Ermon, S.; and Leskovec, J. 2024.
\newblock {TabDiff}: A Unified Diffusion Model for Multi-Modal Tabular Data Generation.
\newblock In \emph{NeurIPS 2024 Table Representation Learning Workshop}.

\bibitem[{Si et~al.(2025)Si, Ou, Qu, Xiang, and Li}]{TabRep}
Si, J.; Ou, Z.; Qu, M.; Xiang, Z.; and Li, Y. 2025.
\newblock TabRep: Training Tabular Diffusion Models with a Simple and Effective Continuous Representation.
\newblock In \emph{ICML Workshop on Foundation Models for Structured Data}.

\bibitem[{Song et~al.(2021)Song, Sohl-Dickstein, Kingma, Kumar, Ermon, and Poole}]{scoresde}
Song, Y.; Sohl-Dickstein, J.; Kingma, D.~P.; Kumar, A.; Ermon, S.; and Poole, B. 2021.
\newblock Score-Based Generative Modeling through Stochastic Differential Equations.
\newblock In \emph{ICLR}.

\bibitem[{{UCI Machine Learning Repository}(2019)}]{stadyn}
{UCI Machine Learning Repository}. 2019.
\newblock {Malware Static and Dynamic Features VxHeaven and Virus Total}.

\bibitem[{{U.S. Department of Agriculture}(2023)}]{education}
{U.S. Department of Agriculture}. 2023.
\newblock Economic Research Service.
\newblock \url{https://www.ers.usda.gov/data-products/county-level-data-sets/county-level-data-sets-download-data/}.

\bibitem[{Valiant(1984)}]{paclearning}
Valiant, L.~G. 1984.
\newblock A Theory of the Learnable.
\newblock \emph{Communications of the ACM}, 27(11): 1134--1142.

\bibitem[{Wang et~al.(2024)Wang, Feng, Dai, Chen, Huang, Ananiadou, Xie, and Wang}]{HARMONIC}
Wang, Y.; Feng, D.; Dai, Y.; Chen, Z.; Huang, J.; Ananiadou, S.; Xie, Q.; and Wang, H. 2024.
\newblock {HARMONIC}: Harnessing {LLM}s for Tabular Data Synthesis and Privacy Protection.
\newblock In \emph{NeurIPS}, 100196--100212.

\bibitem[{Webb(1994)}]{fapproximation}
Webb, A. 1994.
\newblock Functional Approximation by Feed-forward Networks: {A} Least-squares Approach to Generalization.
\newblock \emph{IEEE Transactions on Neural Networks}, 5(3): 363--371.

\bibitem[{Wen et~al.(2022)Wen, Cao, Yang, Subbalakshmi, and Chandramouli}]{causaltgan}
Wen, B.; Cao, Y.; Yang, F.; Subbalakshmi, K.; and Chandramouli, R. 2022.
\newblock Causal-{TGAN}: Modeling Tabular Data Using Causally-Aware {GAN}.
\newblock In \emph{ICLR Workshop on Deep Generative Models for Highly Structured Data}.

\bibitem[{Xu et~al.(2019)Xu, Skoularidou, Cuesta-Infante, and Veeramachaneni}]{ctgan}
Xu, L.; Skoularidou, M.; Cuesta-Infante, A.; and Veeramachaneni, K. 2019.
\newblock Modeling Tabular Data Using Conditional {GAN}.
\newblock In \emph{NeurIPS}, 7335--7345.

\bibitem[{Yang et~al.(2022)Yang, Wu, Cong, Zhang, and He}]{sam}
Yang, J.; Wu, P.; Cong, G.; Zhang, T.; and He, X. 2022.
\newblock SAM: Database Generation from Query Workloads with Supervised Autoregressive Models.
\newblock In \emph{SIGMOD}, 1542--1555.

\bibitem[{Zha et~al.(2025)Zha, Bhat, Lai, Yang, Jiang, Zhong, and Hu}]{dataaugment}
Zha, D.; Bhat, Z.~P.; Lai, K.-H.; Yang, F.; Jiang, Z.; Zhong, S.; and Hu, X. 2025.
\newblock Data-centric Artificial Intelligence: A Survey.
\newblock \emph{ACM Computing Surveys}, 57(5): 129:1--129:42.

\bibitem[{Zhang et~al.(2024)Zhang, Zhang, Srinivasan, Shen, Qin, Faloutsos, Rangwala, and Karypis}]{tabsyn}
Zhang, H.; Zhang, J.; Srinivasan, B.; Shen, Z.; Qin, X.; Faloutsos, C.; Rangwala, H.; and Karypis, G. 2024.
\newblock Mixed-Type Tabular Data Synthesis with Score-based Diffusion in Latent Space.
\newblock In \emph{ICLR}.

\bibitem[{Zhang et~al.(2014)Zhang, Cormode, Procopiuc, Srivastava, and Xiao}]{Privbayes}
Zhang, J.; Cormode, G.; Procopiuc, C.~M.; Srivastava, D.; and Xiao, X. 2014.
\newblock {PrivBayes}: Private Data Release via {Bayesian} Networks.
\newblock In \emph{SIGMOD}, 1423--1434.

\bibitem[{Zhang, Rao, and Agrawala(2023)}]{controlnet}
Zhang, L.; Rao, A.; and Agrawala, M. 2023.
\newblock Adding Conditional Control to Text-to-Image Diffusion Models.
\newblock In \emph{ICCV}, 3836--3847.

\bibitem[{Zhao et~al.(2023)Zhao, Kunar, Birke, Van~der Scheer, and Chen}]{ctabgan+}
Zhao, Z.; Kunar, A.; Birke, R.; Van~der Scheer, H.; and Chen, L.~Y. 2023.
\newblock {CTAB-GAN+}: Enhancing Tabular Data Synthesis.
\newblock \emph{Frontiers in Big Data}, 6: 1296508.

\end{thebibliography}
\clearpage

\clearpage

\appendix

\onecolumn

\section*{Appendix}

\subsection{A Additional Model Details}
\label{app:modeldetails}

\begin{figure}[ht]
    \centering
    \includegraphics[width=\linewidth]{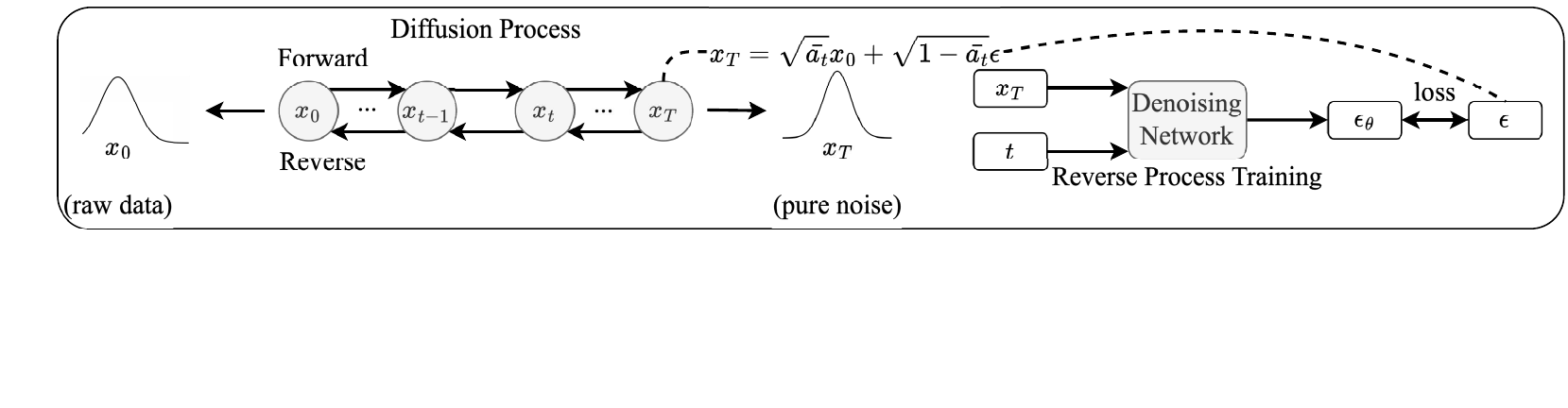}
    \vspace{-20mm}
    \caption{Overview of denoising diffusion model.}
    \label{fig:ddpm}
\end{figure}

\paragraph{Denoising Diffusion Probabilistic Model} Figure~\ref{fig:ddpm} illustrate the forward and the reverse process of DDPM.  The left part illustrates the forward diffusion process, where clean data \( \mathbf{x}_0 \) is gradually corrupted into noise \( \mathbf{x}_T \) via a series of Gaussian perturbations. The closed-form expression \( \mathbf{x}_T = \sqrt{\bar{\alpha}_t}\mathbf{x}_0 + \sqrt{1 - \bar{\alpha}_t}\epsilon \) suggests that \( \mathbf{x}_T \) can be directly sampled at any timestep \( t \). The right part shows the reverse training process, where the denoising network learns to predict the added noise \( \epsilon \) from \( \mathbf{x}_T \) and the timestep \( t \). The training objective minimizes the discrepancy between the predicted noise \( \epsilon_\theta \) and the true noise \( \epsilon \). 

\paragraph{Data Preprocessing of \model} Following existing models~\citep{tabddpm,tabsyn}, for numerical columns,  cells with missing values are filled with the average of that column, and for categorical columns, cells with missing values are filled with a new, ``value-missing'' category.
For tabular data encoding, we follow TabSyn~\cite{tabsyn}, since our model reuses its denoising network. Numerical features are projected with a linear layer while  categorical features are embedded via one-hot encoding followed by an embedding lookup, resulting in a unified feature representation. Subsequently, a Variational Autoencoder (VAE) is employed to encode the feature representation into a latent vector. The latent vector is flattened and used as the initial input $\mathbf{x}_0$ for the diffusion model. This input encoding process is illustrated in Figure~\ref{fig:encoding}. Notably, our overall model design  is input encoding-agnostic -- \model\ can flexibly adapt to different denoising network architectures and employ their corresponding input encoding schemes. For example, when using the denoising network from RelDDPM, categorical variables are encoded using ordinary embedding layers instead of one-hot encoding. 
\begin{figure}[!ht]
    \centering
    \includegraphics[width=1\textwidth]{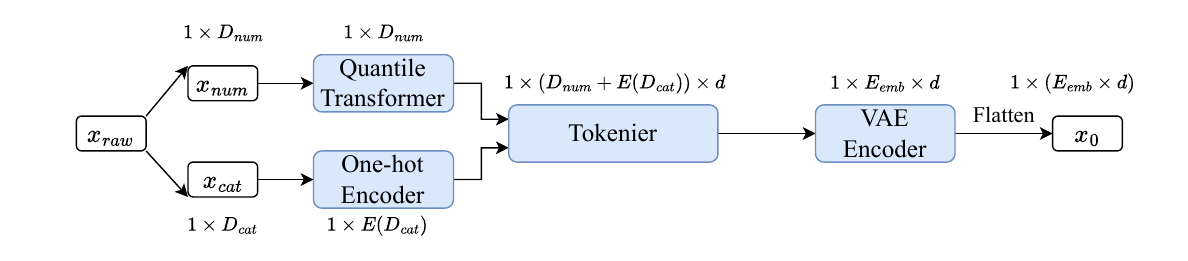} 
    \caption{Input encoding of \model}
    \label{fig:encoding}
\end{figure}

\subsection{B Detailed Proofs} \label{sec:detailedproof}

This section presents a detailed proof for Theorem~\ref{thm:1}.
For simplicity, we prove this result based on the training objective derived from DDPM~\cite{DDPM}.
The proof for the case with the training objective of score-based SDE is analogous and omitted for conciseness. For simplicity, we use $\mathbf{C}_{f}$ to denote the noise-free condition, and $\mathbf{C}_{f} + \mathbf{\tilde{\epsilon}}$ the noise-added condition.

For easy reference, we restate the theorem here:
\begin{theoremrestate}
Let the training objective of \model\ be to minimize $\mathcal{L} = \mathbb{E}_{\mathbf{x}_0, t, \mathbf{\epsilon}, \mathbf{C}_{f}}{ \| \mathbf{\epsilon} - \mathbf{\epsilon_\theta}(\mathbf{x}_t, t, \mathbf{C}_{f}) \|^2}$, and let the training objective with a noise added condition be to minimize $\tilde{\mathcal{L}} = \mathbb{E}_{\mathbf{x}_0, t, \mathbf{\epsilon}, \mathbf{C}_{f}, \mathbf{\tilde{\epsilon}}}{ \| \mathbf{\epsilon} - \mathbf{\epsilon_\theta}(\mathbf{x}_t, t, \mathbf{C}_{f} + \mathbf{\tilde{\epsilon}}) \|^2}$, where $\tilde{\epsilon}$ is a noise drawn from a distribution with mean $0$ and variance $\eta^2$. 
Then, $\tilde{\mathcal{L}} = \mathcal{L} + \eta^2 \mathcal{L}^R$ holds, 
where $\mathcal{L}^R$ is a regularization term in Tikhonov form.
\end{theoremrestate}

\begin{proof}
\textbf{Expansion of the training objectives.} 

For simplicity, define $y(\mathbf{C}_f) = \epsilon_\theta(\mathbf{x}_t, t, \mathbf{C}_{f})$.
Then  
$\mathcal{L} = \mathbb{E}_{\mathbf{x}_0, t, \epsilon, \mathbf{C}_{f}}{ \| \epsilon - \epsilon_\theta(\mathbf{x}_t, t, \mathbf{C}_{f}) \|^2}$
can be expanded as follows:
\begin{align}
      \mathcal{L} &= \int\int\int\int\{ y(\mathbf{C}_f) - \epsilon \}^2p(t,\mathbf{x_0},\epsilon, \mathbf{C}_f)\, dt\, d\mathbf{x_0}\, d\epsilon\, d\mathbf{C}_f \nonumber \\
      &= \int\int\int\int\{ y(\mathbf{C}_f) - \epsilon \}^2p(\epsilon|t,\mathbf{x}_0, \mathbf{C}_f)p(t,\mathbf{x}_0, \mathbf{C}_f)\, dt\, d\mathbf{x}_0\, d\epsilon\, d\mathbf{C}_f \label{eq:E}
\end{align}

Next, we expand $\tilde{\mathcal{L}} = \mathbb{E}_{\mathbf{x}_0, t, \epsilon, \mathbf{C}_{f}, \tilde{\epsilon}}{ \| \epsilon - \epsilon_\theta(\mathbf{x}_t, t, \mathbf{C}_{f} + \tilde{\epsilon}) \|^2}$ as:
\begin{align}
      \tilde{\mathcal{L}} &= \int\int\int\int\int\{ y(\mathbf{C}_f+\tilde{\epsilon}) - \epsilon \}^2p(t,\mathbf{x}_0,\epsilon, \mathbf{C}_f,\tilde{\epsilon})\, dt\, d\mathbf{x}_0\, d\epsilon\,d\mathbf{C}_f\,d\tilde{\epsilon} \nonumber\\
      &= \int\int\int\int\int\{ y(\mathbf{C}_f+\tilde{\epsilon}) - \epsilon \}^2p(t,\mathbf{x}_0,\epsilon, \mathbf{C}_f)p(\tilde{\epsilon})\, dt\, d\mathbf{x}_0\, d\epsilon\, d\mathbf{C}_f\, d\tilde{\epsilon}\nonumber\\
      &\; \text{($\tilde{\epsilon}$ is independent from other variables)}\nonumber \\
      &= \int\int\int\int\int\{ y(\mathbf{C}_f+\tilde{\epsilon}) - \epsilon \}^2p(\epsilon|t,\mathbf{x}_0, \mathbf{C}_f)p(t,\mathbf{x}_0, \mathbf{C}_f)p(\tilde{\epsilon})\, dt\, d\mathbf{x}_0\, d\epsilon\, d\mathbf{C}_f\, d\tilde{\epsilon}\nonumber\\
      &\;(\text{decompose the joint distribution}) \label{eq:Etilde}
\end{align}
\textbf{Some Preliminaries. }
Expanding $y(\mathbf{C}_f+\tilde{\epsilon})$ with Taylor series, and assuming that the noise amplitude is small such that any term of order $O(\tilde{\epsilon}^3)$ or higher can be neglected, we have:
\begin{align}
      y(\mathbf{C}_f+\tilde{\epsilon}) &= y(\mathbf{C}_f) + \sum_i\tilde{\epsilon_i}\left. \frac{\partial y}{\partial \mathbf{C}_{f,i}} \right|_{\tilde{\epsilon}=0} + \frac{1}{2}\sum_i\sum_j\tilde{\epsilon_i}\tilde{\epsilon_j}\left. \frac{\partial^2 y}{\partial \mathbf{C}_{f,i}\partial \mathbf{C}_{f,j}} \right|_{\tilde{\epsilon}=0}+O(\tilde{\epsilon}^3) \label{eq:tayerseries}
\end{align}
By the fact that $\tilde{\epsilon}$ is chosen from a distribution with mean $0$ and variance $\eta^2$, we have the following equalities:
\begin{align}
      & \int \tilde{\epsilon_i}p(\tilde{\epsilon})d\tilde{\epsilon}=0 
      & \int \tilde{\epsilon_i}\tilde{\epsilon_j}p(\tilde{\epsilon})d\tilde{\epsilon}=\eta^2\delta_{ij}
      \text{, where } \delta_{ij} =
        \left\{\begin{array}{ccc}
            1 & i=j\\
            0 & i\neq j
        \end{array}
        \right. \label{eq:eplisontilde}
\end{align}
\textbf{Showing that the noised training objective takes the form of regularization. }
By substituting Equation~\eqref{eq:tayerseries} to Equation~\eqref{eq:Etilde}, we obtain:
\allowdisplaybreaks
\begin{align}
      \tilde{\mathcal{L}} &= \int\int\int\int\int\{ y(\mathbf{C}_f) + \sum_i\tilde{\epsilon_i}\left. \frac{\partial y}{\partial \mathbf{C}_{f,i}} \right|_{\tilde{\epsilon}=0} + \frac{1}{2}\sum_i\sum_j\tilde{\epsilon_i}\tilde{\epsilon_j}\left. \frac{\partial^2 y}{\partial \mathbf{C}_{f,i}\partial \mathbf{C}_{f,j}} \right|_{\tilde{\epsilon}=0}+O(\tilde{\epsilon}^3) - \epsilon \}^2 \nonumber\\
      &p(\epsilon|t,\mathbf{x}_0, \mathbf{C}_f)p(t,\mathbf{x}_0, \mathbf{C}_f)p(\tilde{\epsilon})\, dt\, d\mathbf{x}_0\, d\epsilon\, d\mathbf{C}_f\, d\tilde{\epsilon}\;\text{(substitution)} \label{eq:et1} \\
      &= \int\int\int\int\int\{(y(\mathbf{C}_f) - \epsilon) + \sum_i\tilde{\epsilon_i}\left. \frac{\partial y}{\partial \mathbf{C}_{f,i}} \right|_{\tilde{\epsilon}=0} + \frac{1}{2}\sum_i\sum_j\tilde{\epsilon_i}\tilde{\epsilon_j}\left. \frac{\partial^2 y}{\partial \mathbf{C}_{f,i}\partial \mathbf{C}_{f,j}} \right|_{\tilde{\epsilon}=0}\}^2\nonumber\\
      &p(\epsilon|t,\mathbf{x}_0, \mathbf{C}_f)p(t,\mathbf{x}_0, \mathbf{C}_f)p(\tilde{\epsilon})\, dt\, d\mathbf{x}_0\, d\epsilon\, d\mathbf{C}_f\, d\tilde{\epsilon}\;\nonumber\\
      &\text{(rearrange the terms inside the squared expression and omit the $O(\tilde{\epsilon}^3)$ term)}\label{eq:et2} \\
      &=\int\int\int\int\int\{
      \underbrace{(y(\mathbf{C}_f) - \epsilon)^2}_{(A)} + \underbrace{(\sum_i\tilde{\epsilon_i}\left. \frac{\partial y}{\partial \mathbf{C}_{f,i}} \right|_{\tilde{\epsilon}=0})^2}_{(B)} + \underbrace{(\frac{1}{2}\sum_i\sum_j\tilde{\epsilon_i}\tilde{\epsilon_j}\left. \frac{\partial^2 y}{\partial \mathbf{C}_{f,i}\partial \mathbf{C}_{f,j}} \right|_{\tilde{\epsilon}=0})^2}_{(C)} + \nonumber\\
      &\underbrace{2(y(\mathbf{C}_f) - \epsilon)\sum_i\tilde{\epsilon_i}\left. \frac{\partial y}{\partial \mathbf{C}_{f,i}} \right|_{\tilde{\epsilon}=0}}_{(D)}+ \underbrace{(y(\mathbf{C}_f) - \epsilon)\sum_i\sum_j\tilde{\epsilon_i}\tilde{\epsilon_j}\left. \frac{\partial^2 y}{\partial \mathbf{C}_{f,i}\partial \mathbf{C}_{f,j}} \right|_{\tilde{\epsilon}=0}}_{(E)} + \nonumber\\
      &\underbrace{\sum_i\tilde{\epsilon_i}\left. \frac{\partial y}{\partial \mathbf{C}_{f,i}} \right|_{\tilde{\epsilon}=0}\sum_i\sum_j\tilde{\epsilon_i}\tilde{\epsilon_j}\left. \frac{\partial^2 y}{\partial \mathbf{C}_{f,i}\partial \mathbf{C}_{f,j}} \right|_{\tilde{\epsilon}=0}}_{(F)}\}p(\epsilon|t,\mathbf{x}_0, \mathbf{C}_f)p(t,\mathbf{x}_0, \mathbf{C}_f)p(\tilde{\epsilon})\, dt\, d\mathbf{x}_0\,d\epsilon\,d\mathbf{C}_f\, d\tilde{\epsilon} \;\nonumber\\
      & \text{(expand the square term)} \label{eq:et3} \\ 
      &= \int\int\int\int\int\{ \underbrace{(y(\mathbf{C}_f)-\epsilon)^2}_{(A)} + \underbrace{(\sum_i\tilde{\epsilon_i}\left. \frac{\partial y}{\partial \mathbf{C}_{f,i}} \right|_{\tilde{\epsilon}=0})^2}_{(B)} + \underbrace{2(y(\mathbf{C}_f)-\epsilon)\sum_i\tilde{\epsilon_i}\left. \frac{\partial y}{\partial \mathbf{C}_{f,i}} \right|_{\tilde{\epsilon}=0}}_{(D)}+\nonumber \\
      &\underbrace{(y(\mathbf{C}_f)-\epsilon)\sum_i\sum_j\tilde{\epsilon_i}\tilde{\epsilon_j}\left. \frac{\partial^2 y}{\partial \mathbf{C}_{f,i}\partial \mathbf{C}_{f,j}} \right|_{\tilde{\epsilon}=0}}_{(E)}\}p(\epsilon|t,\mathbf{x}_0, \mathbf{C}_f)p(t,\mathbf{x}_0, \mathbf{C}_f)p(\tilde{\epsilon})\, dt\, d\mathbf{x}_0\, d\epsilon\, d\mathbf{C}_f\, d\tilde{\epsilon} \nonumber \\
      &\text{(omit (C) and (F) as they are of order higher than $O(\tilde{\epsilon}^3)$)} \label{eq:et4} \\
      &=\int\int\int\int\int \underbrace{(y(\mathbf{C}_f)-\epsilon)^2}_{(A)}p(\epsilon|t,\mathbf{x}_0, \mathbf{C}_f)p(t,\mathbf{x}_0, \mathbf{C}_f)p(\tilde{\epsilon})\, dt\, d\mathbf{x}_0\, d\epsilon\, d\mathbf{C}_f\, d\tilde{\epsilon}+\nonumber \\
      &\int\int\int\int\int\underbrace{(\sum_i\tilde{\epsilon_i}\left. \frac{\partial y}{\partial \mathbf{C}_{f,i}} \right|_{\tilde{\epsilon}=0})^2}_{(B)}p(\epsilon|t,\mathbf{x}_0, \mathbf{C}_f)p(t,\mathbf{x}_0, \mathbf{C}_f)p(\tilde{\epsilon})\, dt\, d\mathbf{x}_0\, d\epsilon\, d\mathbf{C}_f\, d\tilde{\epsilon}+\nonumber\\
      &\int\int\int\int\int \underbrace{2(y(\mathbf{C}_f)-\epsilon)\sum_i\tilde{\epsilon_i}\left. \frac{\partial y}{\partial \mathbf{C}_{f,i}} \right|_{\tilde{\epsilon}=0}}_{(D)}p(\epsilon|t,\mathbf{x}_0, \mathbf{C}_f)p(t,\mathbf{x}_0, \mathbf{C}_f)p(\tilde{\epsilon})\, dt\, d\mathbf{x}_0\, d\epsilon\, d\mathbf{C}_f\, d\tilde{\epsilon}\nonumber\\
      &+ \int\int\int\int\int \underbrace{(y(\mathbf{C}_f)-\epsilon)\sum_i\sum_j\tilde{\epsilon_i}\tilde{\epsilon_j}\left. \frac{\partial^2 y}{\partial \mathbf{C}_{f,i}\partial \mathbf{C}_{f,j}} \right|_{\tilde{\epsilon}=0}}_{(E)}\}p(\epsilon|t,\mathbf{x}_0, \mathbf{C}_f)p(t,\mathbf{x}_0, \mathbf{C}_f)p(\tilde{\epsilon})\, dt\, d\mathbf{x}_0\, \nonumber \\
      &d\epsilon\, d\mathbf{C}_f\, d\tilde{\epsilon}\text{(distribute the integrals )} \label{eq:et5}\\
      &= \int\int\int\int \underbrace{(y(\mathbf{C}_f)-\epsilon)^2}_{(A)}p(\epsilon|t,\mathbf{x}_0, \mathbf{C}_f)p(t,\mathbf{x}_0, \mathbf{C}_f)\, dt\, d\mathbf{x}_0\, d\epsilon\, d\mathbf{C}_f\;\text{(omit $\int p(\tilde{\epsilon})\, d\epsilon$ as equals 1)} \nonumber \\
      &+ \eta^2\int\int\int\int\sum_i (\frac{\partial y}{\partial \mathbf{C}_{f,i}})^2p(\epsilon|t,\mathbf{x}_0, \mathbf{C}_f)p(t,\mathbf{x}_0, \mathbf{C}_f)\, dt\, d\mathbf{x}_0\, d\epsilon\, d\mathbf{C}_f\nonumber \\
      &\text{(change $\int\tilde{\epsilon_i}^2p(\tilde{\epsilon})\, d\tilde{\epsilon} = Var(\tilde{\epsilon_i}) = \eta^2$ in (B) from Equality~\eqref{eq:eplisontilde})}\nonumber \\
      &+ \eta^2\int\int\int\int\sum_i (y(\mathbf{C}_f)-\epsilon)\frac{\partial^2 y}{\partial \mathbf{C}_{f,i}^2}p(\epsilon|t,\mathbf{x}_0, \mathbf{C}_f)p(t,\mathbf{x}_0, \mathbf{C}_f)\, dt\, d\mathbf{x}_0\, d\epsilon\, d\mathbf{C}_f \; \nonumber \\
      &\text{(omit (D) term as it contains $\int \tilde{\epsilon_i} P(\tilde{\epsilon})\, d\tilde{\epsilon}$, which is the expectation of $\tilde{\epsilon}$ equal to $0$ by Equality~\eqref{eq:eplisontilde}, } \nonumber \\
      & \text{and change (E) as $\int \tilde{\epsilon_i}\tilde{\epsilon_j}P(\tilde{\epsilon})\, d\tilde{\epsilon}=\eta^2\delta_{ij}=\int \tilde{\epsilon_i}^2P(\tilde{\epsilon})\, d\tilde{\epsilon} = Var(\tilde{\epsilon_i}) = \eta^2$ )}\label{eq:et6} \\
      &= \mathcal{L} + \eta^2 \mathcal{L}^R\,, \label{eq:et7}
      \end{align}
      where
      \begin{align}
       \mathcal{L}^R = \int\int\int\int\sum_i \{ (\frac{\partial y}{\partial \mathbf{C}_{f,i}})^2 + (y(\mathbf{C}_f)-\epsilon)\frac{\partial^2 y}{\partial \mathbf{C}_{f,i}^2} \}p(\epsilon|t,\mathbf{x}_0, \mathbf{C}_f)p(t,\mathbf{x}_0, \mathbf{C}_f)\, dt\, d\mathbf{x}_0\, d\epsilon\, d\mathbf{C}_f. \label{eq:33}
\end{align}
\textbf{Showing that the noised training objective takes a form similar to $L_2$ regularization.}
We define two terms:
\begin{align}
    & \langle\epsilon | t, \mathbf{x}_0, \mathbf{C}_f\rangle = \int \epsilon p(\epsilon | t, \mathbf{x}_0, \mathbf{C}_f) d\epsilon 
    & \langle\epsilon^2 | t, x_0, \mathbf{C}_f\rangle = \int \epsilon^2 p(\epsilon | t, \mathbf{x}_0, \mathbf{C}_f) d\epsilon \label{eq:EandEsquare}
\end{align}
Then, we expand Equation~\eqref{eq:E} as follows:
\begin{align}
    \mathcal{L} &= \int\int\int\int\{ y(\mathbf{C}_f) - \epsilon \}^2p(\epsilon|t,\mathbf{x}_0, \mathbf{C}_f)p(t,\mathbf{x}_0, \mathbf{C}_f)\, dt\, d\mathbf{x}_0\, d\epsilon\, d\mathbf{C}_f \nonumber \\
    &= \int\int\int\int\{ y(\mathbf{C}_f)^2 - 2y(\mathbf{C}_f)\epsilon + \epsilon^2 \}p(\epsilon|t,\mathbf{x}_0, \mathbf{C}_f)p(t,\mathbf{x}_0, \mathbf{C}_f)\, dt\, d\mathbf{x}_0\, d\epsilon\, d\mathbf{C}_f \;\nonumber \\
    &\text{(expand the square term)}\nonumber \\
    &= \underbrace{\int\int\int\int y(\mathbf{C}_f)^2 p(\epsilon|t,\mathbf{x}_0, \mathbf{C}_f)p(t,\mathbf{x}_0, \mathbf{C}_f)\, dt\, d\mathbf{x}_0\, d\epsilon\, d\mathbf{C}_f}_{(A)}  \nonumber \\
    &-\underbrace{2\int\int\int\int y(\mathbf{C}_f)\epsilon p(\epsilon|t,\mathbf{x}_0, \mathbf{C}_f)p(t,\mathbf{x}_0, \mathbf{C}_f)\, dt\, d\mathbf{x}_0\, d\epsilon\, d\mathbf{C}_f}_{(B)} \nonumber \\
    & +\underbrace{\int\int\int\int\ \epsilon^2 p(\epsilon|t,\mathbf{x}_0, \mathbf{C}_f)p(t,\mathbf{x}_0, \mathbf{C}_f)\, dt\, d\mathbf{x}_0\, d\epsilon\, d\mathbf{C}_f}_{(C)}\;\text{(expand terms)}\nonumber \\
    &= \underbrace{\int\int\int y(\mathbf{C}_f)^2 p(t,\mathbf{x}_0, \mathbf{C}_f)\, dt\, d\mathbf{x}_0\, d\mathbf{C}_f}_{(A)} - \underbrace{2\int\int\int y(\mathbf{C}_f)\langle\epsilon | t, \mathbf{x}_0, \mathbf{C}_f\rangle p(t,x_0, \mathbf{C}_f)\, dt\, dx_0\, d\mathbf{C}_f}_{(B)}  \nonumber \\
    &+ \underbrace{\int\int\int\langle\epsilon^2 | t, \mathbf{x}_0, \mathbf{C}_f\rangle P(t,\mathbf{x}_0, \mathbf{C}_f)\, dt\, d\mathbf{x}_0\, d\mathbf{C}_f}_{(C)}  \nonumber \\
    & \text{((A): $\int p(\epsilon|t,\mathbf{x}_0, \mathbf{C}_f)d\epsilon=1$, (B): substitute Equation~\eqref{eq:EandEsquare}, (C): substitute Equation~\eqref{eq:EandEsquare})} \nonumber \\ 
    &= \underbrace{\int\int\int y(\mathbf{C}_f)^2 p(t,\mathbf{x}_0, \mathbf{C}_f)\, dt\, d\mathbf{x}_0\, d\mathbf{C}_f}_{(A)} - \underbrace{2\int\int\int y(\mathbf{C}_f) \langle\epsilon | t, \mathbf{x}_0, \mathbf{C}_f\rangle p(t,\mathbf{x}_0, \mathbf{C}_f)\, dt\, d\mathbf{x}_0\, d\mathbf{C}_f}_{(B)}  \nonumber \\
    &+ \underbrace{\int\int\int \{ Var(\epsilon | t, \mathbf{x}_0, \mathbf{C}_f) + \langle\epsilon | t, \mathbf{x}_0, \mathbf{C}_f\rangle^2 \} p(t,\mathbf{x}_0, \mathbf{C}_f)\, dt\, d\mathbf{x}_0\, d\mathbf{C}_f}_{(C)}  \nonumber \\
    &\text{(rewrite (C) with the equation $Var(\epsilon | t, \mathbf{x}_0, \mathbf{C}_f) = \langle\epsilon^2 | t, \mathbf{x}_0, \mathbf{C}_f\rangle - (\langle\epsilon | t, \mathbf{x}_0, \mathbf{C}_f\rangle)^2$)} \nonumber \\
    &= \int\int\int \{ y(\mathbf{C}_f) - \langle\epsilon | t, \mathbf{x}_0, \mathbf{C}_f\rangle \}^2 p(t,\mathbf{x}_0, \mathbf{C}_f)\, dt\, d\mathbf{x}_0\, d\mathbf{C}_f + \int\int\int Var(\epsilon | t, \mathbf{x}_0, \mathbf{C}_f) p(t,\mathbf{x}_0, \mathbf{C}_f)\, \nonumber \\
    & dt\, d\mathbf{x}_0\, d\mathbf{C}_f\text{(rearrange and combine terms)} \nonumber \\
    &= \int\int\int\int \{ y(\mathbf{C}_f) - \langle\epsilon | t, \mathbf{x}_0, \mathbf{C}_f\rangle \}^2 p(\epsilon|t,\mathbf{x}_0, \mathbf{C}_f)p(t,\mathbf{x}_0, \mathbf{C}_f)\, dt\, d\mathbf{x}_0\, d\epsilon\, d\mathbf{C}_f + \int\int\int\int \nonumber \\
    &\{ \langle\epsilon^2 | t, \mathbf{x}_0, \mathbf{C}_f\rangle-\langle\epsilon | t, \mathbf{x}_0, \mathbf{C}_f\rangle^2 \}p(\epsilon|t,\mathbf{x}_0, \mathbf{C}_f)p(t,\mathbf{x}_0, \mathbf{C}_f)\, dt\, d\mathbf{x}_0\, d\epsilon\, d\mathbf{C}_f\nonumber \\
    &\;\text{(add back $\int p(\epsilon|t,\mathbf{x}_0, \mathbf{C}_f)d\epsilon=1$)} \label{eq:35}
\end{align}
Observe that in Equation~\eqref{eq:35},
only the first term involves the parameters of the neural network (recall that $y(\mathbf{C}_f) =\epsilon_\theta(\mathbf{x}_t, t, \mathbf{C}_{f}) $), while the second term only depends on the ground-truth noise.
Therefore, 
according to Equation~\eqref{eq:35}, $E$ is minimized when 
$y(\mathbf{C}_f) = \langle\epsilon | t, \mathbf{x}_0, \mathbf{C}_f\rangle$.

Moreover, 
recall that $\tilde{\mathcal{L}} = \mathcal{L} + \eta^2 \mathcal{L}^R$ by Equation~\eqref{eq:et7}.
Consider the second term of $\mathcal{L}^R$ in Equation~\eqref{eq:33} and denote it as $\mathcal{L}_2^R$. Then, $\mathcal{L}_2^R$ can be rewritten as follows:
\begin{align}
    \mathcal{L}^R_{2} &= \int\int\int\int\sum_i (y(\mathbf{C}_f)-\epsilon)\frac{\partial^2 y}{\partial \mathbf{C}_{f,i}^2} p(\epsilon|t,\mathbf{x}_0, \mathbf{C}_f)p(t,\mathbf{x}_0, \mathbf{C}_f)\, dt\, d\mathbf{x}_0\, d\epsilon\, d\mathbf{C}_f \nonumber \\
    &= \int\int\int\int\sum_i y(\mathbf{C}_f)\frac{\partial^2 y}{\partial \mathbf{C}_{f,i}^2} p(\epsilon|t,\mathbf{x}_0, \mathbf{C}_f)p(t,\mathbf{x}_0, \mathbf{C}_f)\, dt\, d\mathbf{x}_0\, d\epsilon\, d\mathbf{C}_f -\nonumber \\
    & \int\int\int\int\sum_i \epsilon \frac{\partial^2 y}{\partial \mathbf{C}_{f,i}^2} p(\epsilon|t,\mathbf{x}_0, \mathbf{C}_f)p(t,\mathbf{x}_0, \mathbf{C}_f)\, dt\, d\mathbf{x}_0\, d\epsilon\, d\mathbf{C}_f \;\text{(expand)}\nonumber \\
    &= \int\int\int\int\sum_i y(\mathbf{C}_f)\frac{\partial^2 y}{\partial \mathbf{C}_{f,i}^2} p(\epsilon|t,\mathbf{x}_0, \mathbf{C}_f)p(t,\mathbf{x}_0, \mathbf{C}_f)\, dt\, d\mathbf{x}_0\, d\epsilon\, d\mathbf{C}_f- \nonumber \\
    &\int\int\int\sum_i \langle\epsilon | t, \mathbf{x}_0, \mathbf{C}_f\rangle \frac{\partial^2 y}{\partial \mathbf{C}_{f,i}^2} p(t,\mathbf{x}_0, \mathbf{C}_f)\, dt\, d\mathbf{x}_0\, d\mathbf{C}_f\;\text{(substitute Equation~\eqref{eq:EandEsquare})} \nonumber \\
    &= \int\int\int\int\sum_i y(\mathbf{C}_f)\frac{\partial^2 y}{\partial \mathbf{C}_{f,i}^2} p(\epsilon|t,\mathbf{x}_0, \mathbf{C}_f)p(t,\mathbf{x}_0, \mathbf{C}_f)\, dt\, d\mathbf{x}_0\, d\epsilon\, d\mathbf{C}_f -\int\int\int\int\sum_i  \nonumber \\
    & \langle\epsilon | t, \mathbf{x}_0, \mathbf{C}_f\rangle\frac{\partial^2 y}{\partial \mathbf{C}_{f,i}^2} p(\epsilon|t,\mathbf{x}_0, \mathbf{C}_f) p(t,\mathbf{x}_0, \mathbf{C}_f)\, dt\, d\mathbf{x}_0\, d\epsilon\, d\mathbf{C}_f\;\text{(add back $\int p(\epsilon'|t,\mathbf{x}_0, \mathbf{C}_f)d\epsilon'=1$)} \nonumber \\
    &= \int\int\int\int\sum_i \{ y(\mathbf{C}_f) - \langle\epsilon | t, \mathbf{x}_0, \mathbf{C}_f\rangle \} \frac{\partial^2 y}{\partial \mathbf{C}_{f,i}^2}p(\epsilon|t,\mathbf{x}_0, \mathbf{C}_f)p(t,\mathbf{x}_0, \mathbf{C}_f)\, dt\, d\mathbf{x}_0\, d\epsilon\, d\mathbf{C}_f \nonumber. \\
    &\;\text{(combine terms)}
\end{align}
Therefore, when $\mathcal{L}$ is minimized at $y(\mathbf{C}_f) = \langle\epsilon | t, \mathbf{x}_0, \mathbf{C}_f\rangle$, 
$\mathcal{L}_2^R = 0$. 
As a result, in this case, $\mathcal{L}^R$ can be rewritten to:
\begin{align}
    \mathcal{L}^R &= \int\int\int\int\sum_i (\frac{\partial y}{\partial \mathbf{C}_{f,i}})^2 p(\epsilon|t,\mathbf{x}_0, \mathbf{C}_f)p(t,\mathbf{x}_0, \mathbf{C}_f)\, dt\, d\mathbf{x}_0\, d\epsilon\, d\mathbf{C}_f \nonumber \\
    & = \int\int\int\sum_i (\frac{\partial y}{\partial \mathbf{C}_{f,i}})^2 p(t,\mathbf{x}_0, \mathbf{C}_f)\, dt\, d\mathbf{x}_0\, d\mathbf{C}_f\,,
\end{align}
which is in the Tikhonov form. 
Hence, the proof is completed.
\end{proof}

\subsection{C Additional Details on Experimental Settings}\label{app:exp_settings}

\begin{table}[h]
\centering

\small
\begin{tabular}{lrrrrrrr} 
\toprule
\multicolumn{1}{l}{\textbf{Dataset}} & \multicolumn{1}{l}{$\#$\textbf{Train}} & \multicolumn{1}{l}{$\#$\textbf{Test}} & \multicolumn{1}{l}{$\#$\textbf{Num}} & \multicolumn{1}{l}{$\#$\textbf{Cat}}& \multicolumn{1}{l}{$\#$\textbf{Total}} & \multicolumn{1}{l}{\textbf{Task}}&\textbf{$N / 2^{D}$} \\ 
\midrule
Gesture (\textbf{GE}) & 7,898 & 1,975 & 32 & 0 & 32 & Multi & $9.19e^{-7}$\\
Clean2 (\textbf{CL}) & 5,278 & 1,320 & 166 & 2 & 168 & Binary & $7.05e^{-48}$\\
Malware (\textbf{MA}) & 3,572 & 892 & 241 & 0 & 241 & Binary & $5.05e^{-70}$ \\
Education (\textbf{ED}) & 2,635 & 659 & 52 & 2 & 54 & Binary & $7.31e^{-14}$\\
Unemployment (\textbf{UN}) & 2,621 & 656 & 97 & 2 & 99 & Binary & $2.07e^{-27}$\\
Unemploymentreg (\textbf{UG}) & 2,621 & 656 & 97 & 2 & 99 & Regres & $2.07e^{-27}$\\
Educationreg (\textbf{EG}) & 2,635 & 659 & 52 & 2 & 54 & Regres & $7.31e^{-14}$\\
Stadyn (\textbf{ST})   & 4,998 & 1,250 & 833 & 251 & 1,084 & Binary & $2.5e^{-323}$\\
Arcene  (\textbf{AC})   & 100 & 100 & 10,000 & 1 & 10,001 & Binary & $\approx 0$\\
\bottomrule
\end{tabular}
\caption{Datasets: $\#$Train and $\#$Test refer to the numbers of training and testing rows. $\#$Num and $\#$Cat refer to the numbers of numerical and categorical features, while $\#$Total refers to their sum; ``Task'' specifies the target downstream tasks, where ``Multi'' means multi-class classification, ``Binary'' means binary classification, and ``Regres'' means regression; $N/2^D$ represents the data sparsity, where $N=\#$Train and $D$ is the data dimensionality.} 
\label{tab:datasets} 
\end{table}

\paragraph{Datasets} As summarized in Table~\ref{tab:datasets}, we test our model on the following datasets:
\begin{itemize}
    \item Gesture (\textbf{GE})~\cite{gesture} is a dataset from  human-computer interaction used for gesture phase segmentation which is a multi-class classification task. Each row of the dataset represents the gesture information captured by a device. We use the last categorical column as the class labels while the rest of the columns as the data features for multi-class classification. 
    
    \item Clean2 (\textbf{CL})~\cite{musk2} is a binary classification dataset from the UCI Machine Learning Repository, used to predict whether new molecules are musk or non-musk. Each row of the dataset represents the exact shape, or conformation, of a molecule. We use the last categorical column as the class labels while the rest of the columns as the data features for  classification.
    
    \item Malware (\textbf{MA})~\cite{malware} is another binary classification dataset from the UCI Machine Learning Repository, used to predict whether a software is malware. Each row of the dataset represents software permission information. We use the last column as the class labels while the rest of the columns as the data features for classification.
    
    \item Education (\textbf{ED})~\cite{education} is a dataset from the US Department of Agriculture that provides statistics on the educational attainment of adults aged 25 and older in the USA, across states and counties, from 1970 to 2022. We use this dataset for a binary classification task by transforming the column ``\emph{Percent of adults with a bachelor's degree or higher, 2018-22}'' based on the median value: values greater than the median are labeled as Class 1, while others are labeled as Class 0.
    
    \item Unemployment (\textbf{UN})~\cite{education} is a dataset from the US Department of Agriculture that contains median household income for the the states and counties in the USA from 2000 to 2022. We use this dataset for a binary classification task by converting column ``\textit{County household median income as a percent of State total median household income, 2021}'' into binary labels, where $1$ means ``YES'' and $0$ otherwise. 
    
    \item Unemploymentreg (\textbf{UG})~\cite{education} is the same  unemployment dataset as above. Now we use the raw values of the column ``\emph{County household median income as a percent of State total median household income, 2021}'' to form a regression task.
    
    \item Educationreg (\textbf{EG})~\cite{education} is the same Education dataset as above. Now we use the raw values from the column ``\emph{Percent of adults with a bachelor's degree or higher, 2018-22}''  to form a regression task.
    \item Stadyn  (\textbf{ST})~\cite{stadyn} is a binary classification dataset from the UCI Machine Learning Repository, used for malware classification. Each row of the dataset represents static and dynamic malware features such as Hex dump.

    \item Arcene  (\textbf{AC})~\cite{arcene} is a binary classification dataset from the UCI Machine Learning Repository, used to distinguish cancer versus normal patterns from mass-spectrometric data. It is a highly sparse dataset containing 10,000 features (7,000 real and 3,000 distractors) but only 100 samples. Each row of the dataset contains multiple features representing the abundance of proteins in human sera with specific mass values.
\end{itemize}

\paragraph{Competitors} We compare \model\ with six models including state-of-the-art (SOTA) diffusion-based tabular data synthesis models:
\begin{itemize}
    \item \textbf{SMOTE}~\cite{smote} is a traditional yet highly effective oversampling technique proposed to address the class imbalance issue by generating synthetic samples for a minority class. It creates new data samples by interpolating between existing samples of the target minority class to increase their representation in the dataset.

    \item \textbf{TVAE}~\cite{ctgan} is a variational autoencoder-based model for synthesizing tabular data.
    
    \item \textbf{CTGAN}~\cite{ctgan} is a GAN-based model that uses conditional information and different normalizations to generate data samples addressing challenges brought by class imbalance and complex data distribution.
    
    \item \textbf{TabDDPM}~\cite{tabddpm} is one of the first models that extend diffusion models to tabular data synthesis. It proposes  different diffusion processes for different types (i.e., numerical and categorical) of data from a table. 
    
    \item \textbf{RelDDPM}~\cite{relddpm} is a recent diffusion-based model based on classifier-guidance. It trains an additional classifier to provide supervision signals to guide the data synthesis process towards a given conditional direction.
    
    \item \textbf{TabSyn}~\cite{tabsyn} is the SOTA model for tabular data synthesis. It uses a B-VAE module to learn an embedding for each sample in a data table. Then, it employs a score-based diffusion model for tabular data synthesis using the learned embeddings.
\end{itemize}

\paragraph{Implementation Details} All experiments were run three times on a virtual machine with a 12-core Intel(R) Xeon(R) Gold 6326 CPU (2.90 GHz), $32$ GB of RAM, and an NVIDIA A100 GPU. We implement \model\ and all competitor models using Python. For the competitors, we adopt the same hyperparameter settings as specified in the TabSyn paper~\citep{tabsyn}.  
We train \model\ with the AdamW optimizer~\citep{adamw} 
using a learning rate of $0.0018$ and weight decay of $0.00001$ with $30,000$ steps, and noise scale 0.005.

\subsection{D Additional Experimental Results}

\paragraph{Additional Ablation Study Results}\label{sec:additionalablation}
Table~\ref{tab:differentmodules2} shows results in AUC and RMSE of CtrTab and six alternative variants of TabSyn as well as  \model-w/o-lastfusion, to complement the F1 and R2 scores show in Table~\ref{tab:differentmodules}.
The comparative patterns in AUC and RMSE are similar to those in F1 and R2, where our \model\ model outperforms all six variants of TabSyn as well as w/o-lastfusion on all datasets tested. This further confirms the effectiveness of our proposed control module in \model, the staged training strategy, and as the last $h_{fusion}$ connection.

{\small
\begin{table}[ht]
\centering

\begin{tabular}{lccccccc}
\toprule
\textbf{Method} & \textbf{GE} & \textbf{CL} & \textbf{MA} & \textbf{ED} & \textbf{UN} & \textbf{UG} & \textbf{EG}  \\
\midrule
\textbf{Metric} & AUC$\uparrow$ & AUC$\uparrow$ & AUC$\uparrow$ & AUC$\uparrow$ & AUC$\uparrow$ & RMSE$\downarrow$ & RMSE$\downarrow$  \\
\midrule
Train$\times 2$ & 0.795 & 0.979 & 0.996 & 0.958 & 0.818 & 0.285 & 0.151 \\
Data$\times 2$  & 0.804 & 0.990 & 0.998 & 0.963 & 0.916 & 0.192 & 0.143 \\
Model$\times 2$ & 0.827 & 0.990 & 0.997 & 0.981 & 0.919 & 0.256 & 0.127 \\
NoiseCond & 0.785 & 0.980 & 0.998 & 0.966 & 0.868 & 0.289 & 0.159 \\
Dropout-Reg & 0.764 & 0.976 & 0.996 & 0.958 & 0.877 & 0.391 & 0.151  \\
JointTrain & 0.835 & 0.990 & 0.992 & 0.960 & 0.847 & 0.360 & 0.127  \\
w/o-lastfusion & 0.886 & 0.999 & \textbf{0.999} & 0.987 & 0.974 & 0.338 & \textbf{0.114} \\
\textbf{\model} & \textbf{0.890} & \textbf{1} & \textbf{0.999} & \textbf{0.988} & \textbf{0.979} & \textbf{0.069} & \textbf{0.114} \\
\bottomrule
\end{tabular}
\caption{Ablation study results in AUC and RMSE.}
\label{tab:differentmodules2}
\end{table}
}

\paragraph{Comparison of Training and Inference Times}
\begin{table}[h!]
   \centering


\begin{tabular}{lcccccc}
\toprule
\textbf{Model} & \textbf{GE Train} & \textbf{GE Inference} & \textbf{UN Train} & \textbf{UN Inference} & \textbf{AC Train} & \textbf{AC Inference}\\
\midrule
TabSyn & 1584.58s & 2.78s & 1617.68s & 1.86s & 1757s & 85s\\
CtrTab & 2744.58s  & 4.12s & 2936.18s & 2.47s & 3708s & 88s\\
\bottomrule
\end{tabular}
\caption{Training and inference times of TabSyn and \model\ on GE and UN datasets.}
\label{tab:runtime}
\end{table}
We compare the training and inference times of \model\ with a diffusion-based model TabSyn, since our model is implemented based on TabSyn. Diffusion models have a time complexity of $\mathcal{O}(TDLH)$; $T$: steps; $D$: input dim; $L$: depth; $H$: hidden size. Although \model\ contains approximately $1.8$ times more parameters than TabSyn, its runtime remains in the same order of magnitude, as shown in Table~\ref{tab:runtime}. We report both training and inference times on the GE, UN and AC  datasets for conciseness -- these are the first 
classification, regression and highest dimensional datasets in Table~\ref{tab:mlperformance}.

\paragraph{Results on the Impact of Training Set Size}

Table \ref{tab:nsamplestest2} reports results on the Adult\textsuperscript{1} (ADU, a widely used conventional dataset with 48,842 instances and 14 columns) and MA (the most sparse among our high-dimensional datasets) datasets, showing how performance varies with training set size as measured by the average of column-wise distributions and pairwise column correlations (scores in the range $[0, 1]$; higher is better), using the two most competitive baselines as shown earlier.

We follow previous work~\citep{tabsyn} and this metric is proposed by SDMetrics. Column-wise density estimation uses Kolmogorov-Sirnov Test (KST) for numerical columns and the Total Variation Distance (TVD) for categorical columns. Pair-wise correlation uses Pearson correlation for numerical columns and contingency similarity for categorical columns. 

Across all training set sizes, \model\ attains the highest scores and maintains a substantial margin over the two competitors. Moreover, its performance degrades only marginally as the amount of real data shrinks, demonstrating \model's strong robustness to limited supervision. This resilience is most evident on the MA dataset, which is high-dimensional and highly sparse: even when only 10\% of the original training data (357 samples) is used, \model\ still captures meaningful patterns and produces synthetic data that closely mirrors the real distribution.

\begin{table}[H]
    \centering

    {\small
    \begin{tabular}{lcccc}
        \toprule
         \textbf{Method} & \textbf{10\%} & \textbf{30\%} & \textbf{50\%} & \textbf{100\%} \\
        \midrule
        \multicolumn{5}{l}{\textbf{ADU Dataset}} \\
        \midrule
        \textbf{SMOTE}      & 94.0\% & 93.5\% & 94.8\% & 97.5\% \\
        \textbf{TabSyn}     & 94.5\% & 97.6\% & 98.1\% & 98.9\% \\
        \textbf{\model}     & \textbf{99.0\%} & \textbf{99.6\%} & \textbf{99.6\%} & \textbf{99.7\%} \\
        \midrule
        \multicolumn{5}{l}{\textbf{MA Dataset}} \\
        \midrule
        \textbf{SMOTE}      & 97.3\% & 97.9\% & 99.4\% & 99.3\% \\
        \textbf{TabSyn}     & 96.8\% & 97.4\% & 96.4\% & 96.7\% \\
        \textbf{\model}     & \textbf{99.3\%} & \textbf{99.2\%} & \textbf{99.8\%} & \textbf{99.8\%} \\
        \bottomrule
    \end{tabular}
    }
    \caption{Impact of training set size on data density performance for the ADU and MA datasets.}
    \label{tab:nsamplestest2}
\end{table}

\paragraph{Results on Lower-Dimensional Datasets}
To verify that our module is not limited to high-dimensional settings, we evaluate it on the same datasets used by the TabSyn paper~\citep{tabsyn}. These datasets are Adult (ADU, 48,842 instances and 14 columns), Default (DEF, 30,000 instances and 25 columns), Shoppers (SHO, 12,330 instances and 18 columns), Magic (MAG, 19,019 instances and 11 columns), Beijing (BEI, 43,824 instances and 12 columns), News (NEW, 39,644 instances and 48 columns). We follow the training/test set settings in the TabSyn paper~\citep{tabsyn} and compare TabSyn against \model, running \model\ with its default hyper-parameters and a noise scale of $0.005$. Table~\ref{tab:downstreamregular} presents machine learning test (same procedure as before) results using data synthesized by the two models. As the table shows, \model\ also delivers strong performance on lower-dimensional datasets. This validates the applicability of \model\ beyond its targeted sparse, high-dimentional settings.

\begin{table}[h]
\centering

\begin{tabular}{ccccccc|c}
\toprule
\textbf{Methods} & \textbf{ADU} & \textbf{DEF} & \textbf{SHO} & \textbf{MAG} & \textbf{BEI} & \textbf{NEW} & \textbf{Average Gap} \\
\midrule
 & AUC $\uparrow$ & AUC $\uparrow$ & AUC $\uparrow$ & AUC $\uparrow$ & RMSE $\downarrow$ & RMSE $\downarrow$ & $\% \downarrow$\\
\midrule
Real    &   0.927   &  0.770   &  0.926  &  0.946    &  0.432   & 0.842 & 0\%   \\
\midrule
TabSyn    &   0.911   &  0.762   &  0.920  &  0.934    &  0.685   & 0.854 & 10.78\%   \\
\model     &  \textbf{0.917}   &  \textbf{0.768}   &   \textbf{0.928}   &  \textbf{0.947}   &  \textbf{0.531}   & \textbf{0.826} & \textbf{4.04\%} \\
\bottomrule
\end{tabular}
\caption{Downstream results on non-high-dimensional datasets}
\label{tab:downstreamregular}
\end{table}

\paragraph{Results on Data Visualization}
\begin{figure}[!ht]
    \centering
    \includegraphics[width=0.8\textwidth]{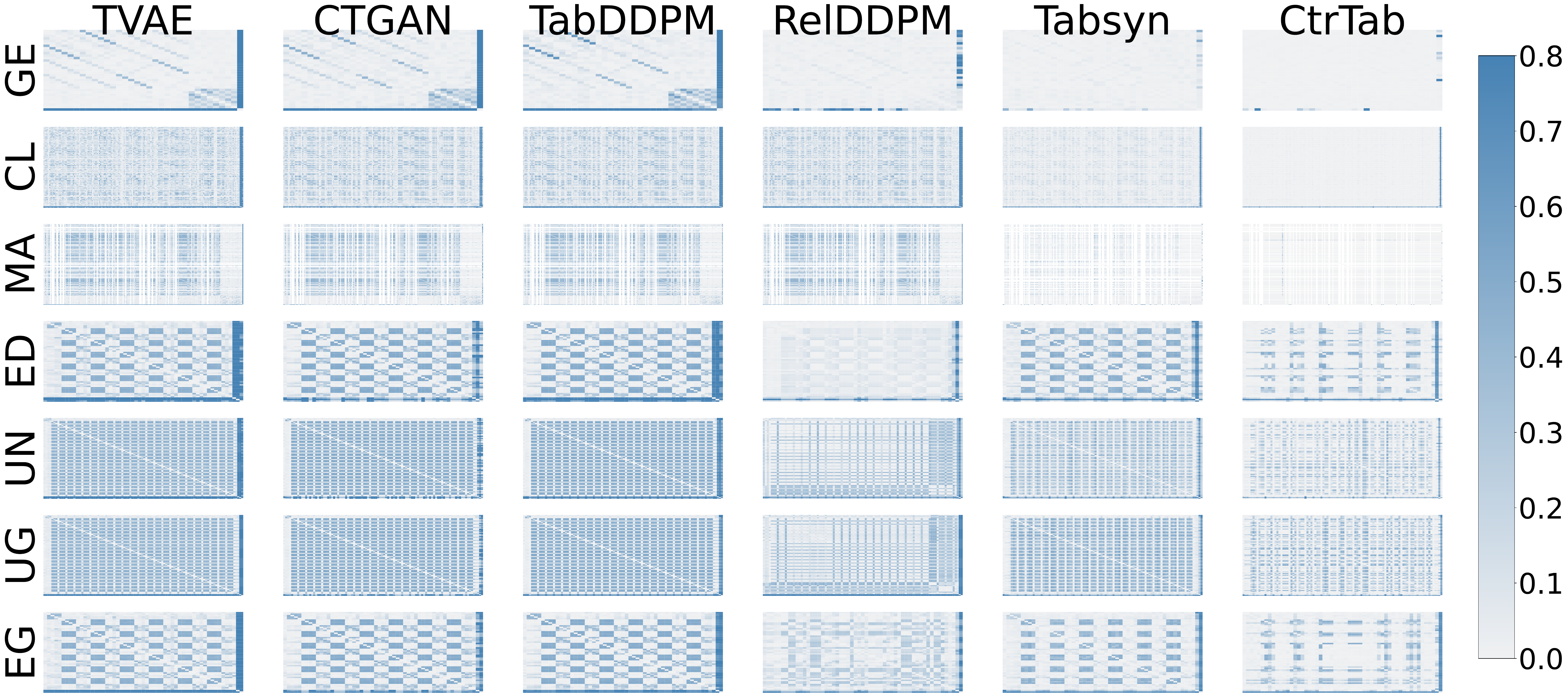} 
    \caption{Heatmaps of the pair-wise column correlation of synthetic data v.s. the real data.}
    \label{fig:heatmap}
\end{figure}

In Figure \ref{fig:heatmap}, we present the divergence heatmap of pairwise column correlations between real datasets and datasets synthesized by different models. The color intensity reflects the correlation error rate for each pair of columns, where lighter colors indicate smaller discrepancies. Notably, these differences are substantially reduced when comparing the baseline Tabsyn model to \model, demonstrating that our proposed control module effectively improves the diffusion model’s capacity to capture the distribution of high-dimensional data. This improvement arises from the explicit guidance provided by the control module, which enhances the model’s ability to represent pairwise column correlations more precisely, especially under sparse and high-dimensional regimes.

\paragraph{Results on Distance to Closest Records}\label{sec:dcr}
Following the TabSyn paper, to ensure that \model\ does not simply replicate the training data, we further measure the Distance to Closest Records (DCR) scores of \model\ and the best baseline SMOTE, which denote the probability that a synthesized sample finds its closest (by $L_1$ distance) sample in the training set (rather than the testing set). For an ideal model, its DCR scores should be close to 0.5 (i.e., the synthesized samples have good chances to be similar to both the training and the test sets). To enable a consistent comparison between DCR values smaller and greater than 0.5, we adopt the normalized version, NDCR = $|DCR - 0.5|$. Table~\ref{tab:privacy} reports the NDCR. 
\model\ outperforms SMOTE on all but one datasets, with  NDCR values that are $1.5\%$ smaller on average. These confirm that \model\ does not simply replicate the training data.   
For the other baselines, their NDCR values are lower since their synthesized data do not fit the original data distribution as shown above. We omit those values for conciseness.  

\begin{table}[ht]
    \centering
    
    {\small
    \setlength{\tabcolsep}{2pt}
    \begin{tabular}{lcccccccc}
        \toprule
        \textbf{Method} & \textbf{GE} & \textbf{CL} & \textbf{MA} & \textbf{ED} & \textbf{UN} & \textbf{UG} & \textbf{EG} & \textbf{Average}\\
        \midrule
        SMOTE             & \textbf{0.477} & 0.500 & 0.342 & 0.500 & 0.500 & 0.500 &0.500 & 0.474\\
        \textbf{\model}     & 0.497 & \textbf{0.479} & \textbf{0.319} & \textbf{0.492} & \textbf{0.499} & \textbf{0.491} & \textbf{0.492} & \textbf{0.467} \\
        \bottomrule
    \end{tabular}
    }
    \caption{NDCR results (smaller values are preferred).}
    \label{tab:privacy}
\end{table}

\paragraph{Results on the Impact of Noise Types on \model}\label{sec:diffenentnoisetype} 
We investigate how alternative noise distributions injected into the control module of \model\ impacts the quality of the data synthesized. Replacing the default Laplace noise with Normal or Uniform noise, we tune the perturbation scale to match the resulting (normalized) NDCR, and we observe the machine learing test results of the different model variants. Table \ref{tab:differentnoisetypes} summarizes NDCR and downstream machine learning test results. Across all dataset, Laplace noise delivers the most stable and in general the best machine learning effectiveness results—achieving the lowest or tied NDCR while producing high machine learning test scores—verifying the effectiveness of its use in our model.

\begin{table}[t]
\centering
\begin{tabular}{lcccc}
\toprule
\textbf{Datasets} & \textbf{Noise Type} & \textbf{F1}$\uparrow$ & \textbf{AUC}$\uparrow$ & \textbf{NDCR}$\downarrow$ \\
\midrule
GE          & Normal  & 0.619 & 0.888 & 0.499 \\
            & Laplace & \textbf{0.635} & \textbf{0.890} & 0.497 \\
            & Uniform & \textbf{0.635} & \textbf{0.890} & 0.499 \\
CL          & Normal  & 0.978 & 0.999 & 0.480 \\
            & Laplace & \textbf{0.983} & \textbf{1.000} & 0.479 \\
            & Uniform & 0.978 & 0.999 & 0.477 \\
MA          & Normal  & \textbf{0.994} & \textbf{1} & 0.319 \\
            & Laplace & \textbf{0.994} & 0.999 & 0.319 \\
            & Uniform & 0.992 & 0.999 & 0.319 \\
ED          & Normal  & 0.912 & \textbf{0.988} & 0.492 \\
            & Laplace & \textbf{0.918} & \textbf{0.988} & 0.492 \\
            & Uniform & 0.917 & 0.985 & 0.492 \\
UN          & Normal  & 0.890 & \textbf{0.980} & 0.499 \\
            & Laplace & \textbf{0.898} & 0.979 & 0.499 \\
            & Uniform & \textbf{0.898} & 0.978 & 0.498 \\
\midrule
\textbf{Datasets} & \textbf{Noise Type} & \textbf{R2}$\uparrow$ & \textbf{RMSE}$\downarrow$ & \textbf{NDCR}$\downarrow$ \\
\midrule
UG          & Normal  & 0.988 & 0.083 & 0.493 \\
            & Laplace & \textbf{0.991} & \textbf{0.069} & 0.491 \\
            & Uniform & 0.986 & 0.091 & 0.491 \\
EG          & Normal  & 0.606 & 0.122 & 0.491 \\
            & Laplace & \textbf{0.657} & \textbf{0.114} & 0.492 \\
            & Uniform & 0.639 & 0.117 & 0.491 \\
\bottomrule
\end{tabular}
\caption{Comparison of using Normal noise, Laplace noise and Uniform noise for \model\ on different datasets.
Under the condition of maintaining similar (normalized) NDCR values, we test the impact of the different types of noise on the data synthesis results using the machine learning test.}
\label{tab:differentnoisetypes}
\end{table}

\end{document}